\documentclass{article}

\usepackage{microtype}
\usepackage{graphicx}
\usepackage{subcaption}
\usepackage{booktabs} 
\usepackage{dblfloatfix}
\usepackage{placeins}
\usepackage{needspace}
\usepackage{afterpage}

\usepackage{booktabs}
\usepackage{multirow}
\usepackage{adjustbox}
\usepackage{booktabs}
\usepackage{graphicx}
\usepackage[table]{xcolor}
\usepackage{tabularx}
\usepackage{ulem}
\usepackage{arydshln}

\newcolumntype{C}{>{\centering\arraybackslash}X}

\usepackage[hidelinks]{hyperref}
\usepackage[accepted]{icml2026}

\usepackage{amsmath}
\usepackage{amssymb}
\usepackage{mathtools}
\usepackage{amsthm}
\usepackage{capt-of}
\usepackage{placeins}

\usepackage{xcolor}

\usepackage[capitalize,noabbrev]{cleveref}

\theoremstyle{plain}
\newtheorem{theorem}{Theorem}[section]
\newtheorem{proposition}[theorem]{Proposition}

\theoremstyle{definition}

\newtheorem{assumption}[theorem]{Assumption}
\theoremstyle{remark}

\usepackage[textsize=tiny]{todonotes}

\icmltitlerunning{On the Salience of Low-Probability Tokens for AI-Generated Text Detection: A Multiscale Uncertainty Perspective}

\begin{document}

\twocolumn[
  \icmltitle{On the Salience of Low-Probability Tokens for AI-Generated Text Detection: \\A Multiscale Uncertainty Perspective}
  \icmlsetsymbol{equal}{*}

  \begin{icmlauthorlist}
    \icmlauthor{Yikai Guo}{a}
    \icmlauthor{Bin Wang}{a}
    \icmlauthor{Xilai Fan}{b}
    \icmlauthor{Wenjun Ke}{c}
    \icmlauthor{Haoran Luo}{d}
  \end{icmlauthorlist}

  \icmlaffiliation{a}{Beijing Institute of Computer Technology and Application, Beijing, China}
  \icmlaffiliation{b}{Academy of Mathematics and Systems Science, Chinese Academy of Sciences, Beijing, China}
  \icmlaffiliation{c}{Southeast University, Nanjing, China}
 \icmlaffiliation{d}{College of Computing and Data Science, Nanyang Technological University, Singapore}

  \icmlcorrespondingauthor{Xilai Fan}{fanxilai@lsec.cc.ac.cn}

  \icmlkeywords{Machine Learning, ICML}

  \vskip 0.3in
]



\printAffiliationsAndNotice{}  

\begin{abstract}

AI-generated text increasingly blends with human writing, raising practical risks such as misinformation, academic misuse, and corpora contamination. While statistical detectors are appealing for efficiency and generalization, they suffer from two key limitations. (i) Boilerplate dominance, boilerplate tokens shared across human and LLM writing can overwhelm discriminative signals. (ii) Brittle point estimates, relying on a single probability score yields unstable decisions under adversarial manipulations.
To address these issues, we propose \textbf{Uncertainty}, a multiscale uncertainty estimator that focuses on informative low-probability tokens, which more clearly expose distributional discrepancies. Locally, it alleviates boilerplate dominance by averaging the log-probabilities of low-probability tokens; globally, it reduces brittleness by capturing the distributional shape of this low-probability region via R{\'e}nyi entropy.
We further extend the detector to \textbf{Uncertainty++} via conditional independent sampling, yielding a more stable uncertainty estimation. Experiments across seven datasets and sixteen LLMs demonstrate high effectiveness, generalization, and robustness. Our code is available at \href{https://github.com/guoyikai2000/Uncertainty-AIGT}{github.com/guoyikai2000/Uncertainty-AIGT}.

\end{abstract}

\section{Introduction}
As LLMs continue to evolve, AI-generated text is increasingly difficult to distinguish from human writing~\cite{chen25online}. This proliferation creates tangible risks, including misinformation~\cite{pudasaini25benchmarking}, academic misuse~\cite{Delving25Kobak}, and contaminated web corpora~\cite{drayson25machine}. The scale, diversity, and manipulability of AI-generated text make it necessary to develop detectors that are efficient~\cite{POGER24Shi}, generalizable~\cite{zhou25kill}, and robust~\cite{macko-etal-2025-multisocial}.

Existing approaches to AI-generated text detection can be broadly grouped into three families: watermarking, fine-tuning, and statistical methods~\cite{masrour-etal-2025-damage}. Watermarking methods embed a detectable signature during generation that a verifier can later recognize~\cite{mao25watermarking}; however, tests on popular public LLMs did not detect statistically significant evidence of watermarking, which can limit the practical applicability of watermark-based detection~\cite{gloaguen25blackbox}. Fine-tuning methods train discriminative detectors on curated corpora and can perform well when test-time conditions match the training distribution, but constructing and maintaining such datasets is costly; moreover, these detectors often generalize poorly when the generator or domain changes~\cite{chen25divscore}. In contrast, statistical methods score a candidate text with one or more reference language models and summarize token-level statistics into features, enabling efficient inference and typically stronger generalization~\cite{xu25lastde}.

\begin{figure}[t]
    \centering
    \includegraphics[width=1.0\columnwidth]{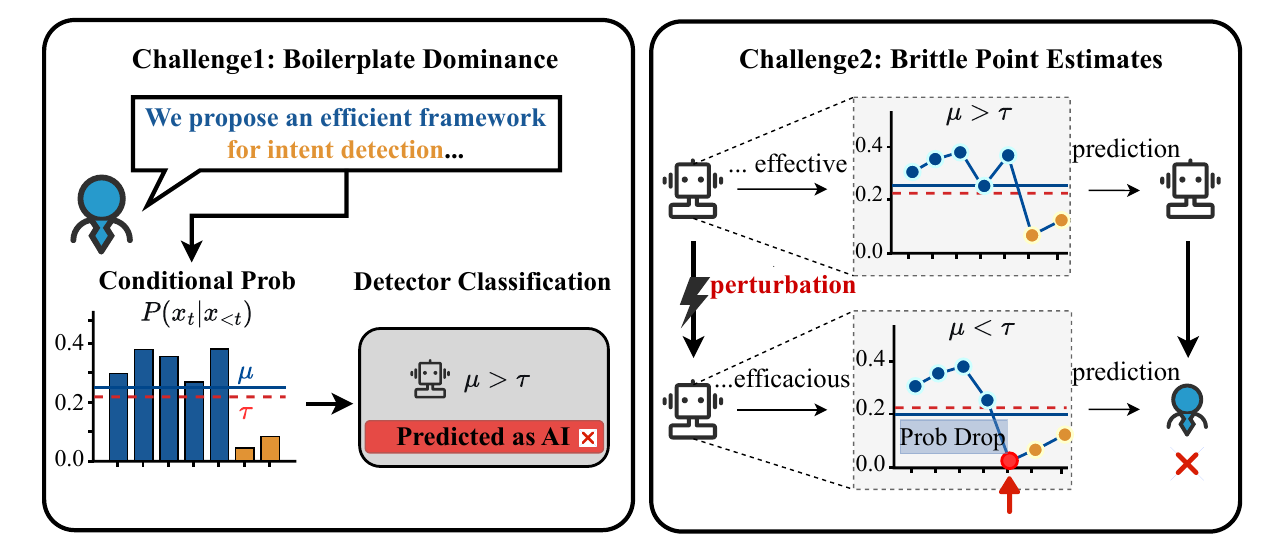}
    \caption{Two challenges faced by statistical methods which use proxy model to score. \textbf{Left: Boilerplate dominance.} Sequence-level aggregation of token-wise conditional probabilities can be dominated by high-probability boilerplate spans that are common to both human and LLM writing, leading to misclassification (e.g., a human-written sentence predicted as AI-generated when the average score $\mu$ exceeds a threshold $\tau$). \textbf{Right: Brittle point estimates.} Collapsing each conditional distribution to a single probability score discards distributional shape, making decisions sensitive to  surface-level perturbations (e.g., decoding or paraphrasing), which can substantially shift $\mu$ and flip the prediction.}
\label{fig:intro}
\end{figure}

Nevertheless, existing statistical methods still face two challenges: boilerplate dominance and brittle point estimates. First, aggregating likelihood signals over all tokens can be dominated by high-probability boilerplate shared by humans and LLMs, such as routine wording, common phrases, and generic transitions used to connect points or add context.  As illustrated in Fig.~\ref{fig:intro}, the human-written text \textit{We propose an efficient framework for intent detection} contains boilerplate spans \textit{propose an efficient framework} that also appear frequently in AI-generated text. Because a proxy model assigns similarly high likelihood to such ubiquitous spans across both classes, averaging over the full sequence can dilute the contribution of genuinely discriminative tokens.
Second, many existing approaches rely on observed token probability as a point estimate, collapsing the underlying conditional distribution into a single scalar and discarding information about its shape. As a result, these detectors can be sensitive to manipulations such as changes in decoding strategy or paraphrasing. As shown in Fig.~\ref{fig:intro}, even surface-level edits can substantially shift these point-estimate scores.

To address these  challenges, we propose a multiscale uncertainty estimator that captures both local and global properties of conditional distributions. To mitigate the dominance of high-probability boilerplate, we avoid treating all tokens equally. Instead, we rank token positions by the observed-token log-probability and compute statistics only on the bottom tail, thereby emphasizing informative low-probability regions. To improve robustness beyond point estimates, we further incorporate distributional information by computing R{\'e}nyi entropy of the conditional distribution at each position and averaging it over the same tail positions, which summarizes global uncertainty in distribution shape. The R{\'e}nyi order $\alpha$ offers a controllable bias toward head or tail mass in the vocabulary dimension, complementing the tail token selection. We then fuse these two signals into a unified detection score, yielding \textbf{Uncertainty}. Leveraging conditional independent sampling, we further extend our method to \textbf{Uncertainty++}, producing a more stable variant. Experiments across seven datasets and sixteen generators demonstrate high effectiveness, generalization, and robustness of our methods.

\paragraph{Contributions.}Our main contributions are: (a) identifying and empirically validating the \emph{low-probability hypothesis}, namely that low-probability tokens carry more discriminative evidence for AI-generated text detection than high-probability boilerplate; and (b) proposing a multiscale uncertainty estimator inspired by this hypothesis that fuses the tail mean log-probability with R{\'e}nyi entropy averaged over the same tail positions into a unified detection score.

\paragraph{Conflict of Interest Disclosure.} The authors declare no financial or other substantive conflicts of interest.

\section{Preliminaries}

\paragraph{Problem Formulation.}
We study zero-shot AI-generated text detection as a binary classification problem: given a candidate text $\mathbf{x}$, determine whether it is human-written or AI-generated, without access to labeled training data.

We consider both white-box and black-box settings.
In the white-box setting, the detector has access to a source model that generated the text.
In the black-box setting, the source model is inaccessible, so the detector performs inference through a proxy model rather than the true generator. This setting is motivated by the empirical observation that modern language models, despite architectural differences, exhibit shared statistical regularities that distinguish AI-generated text from human writing~\cite{xu25lastde}.

\paragraph{Conditional Probability Distribution.}
Let $\mathbf{x} = \{x_i\}_{i=0}^{n-1}$ denote a sequence of text tokens of length $n$, where each token $x_i$ belongs to a vocabulary $\mathcal{V}$ of size $|\mathcal{V}| = D$.
We use $x_{<i} = \{x_i\}_{i=0}^{i-1}$ to denote the prefix at position~$i$.
At each position $i$, this full distribution $p_\theta(\cdot \mid x_{<i})$ characterizes the model's uncertainty over the next token, which we refer to as the \emph{conditional probability distribution}.
The joint probability of a sequence factorizes as
$p_\theta(\mathbf{x}) = \prod_{i=1}^{n-1} p_\theta(x_i \mid x_{<i})$. In the white-box setting, the detector has access to a source model that generated the text. 

\paragraph{Low-Probability Tokens.}
We define \emph{low-probability tokens} relative to the distribution of conditional probabilities within each sequence.
Specifically, given a sequence $\mathbf{x}$ of length $n$ and a percentile level $\rho \in (0,1]$, let $\mathcal{S}_\rho \subseteq \{1, \ldots, n-1\}$ denote the index set of positions whose chosen conditional probabilities $p_\theta(x_i \mid x_{<i})$ fall in the bottom $\rho$ fraction among all positions in that sequence.
We call the tokens at positions in $\mathcal{S}_\rho$ the \emph{low-probability tokens} at level $\rho$.
Such tokens correspond to positions where the model finds the actual token choice relatively surprising compared to other positions in the same sequence.

\section{Method}

\begin{figure*}[t]
\centering
\includegraphics[width=0.88\textwidth, trim=3.5cm 2cm 3.2cm 2.4cm, clip]{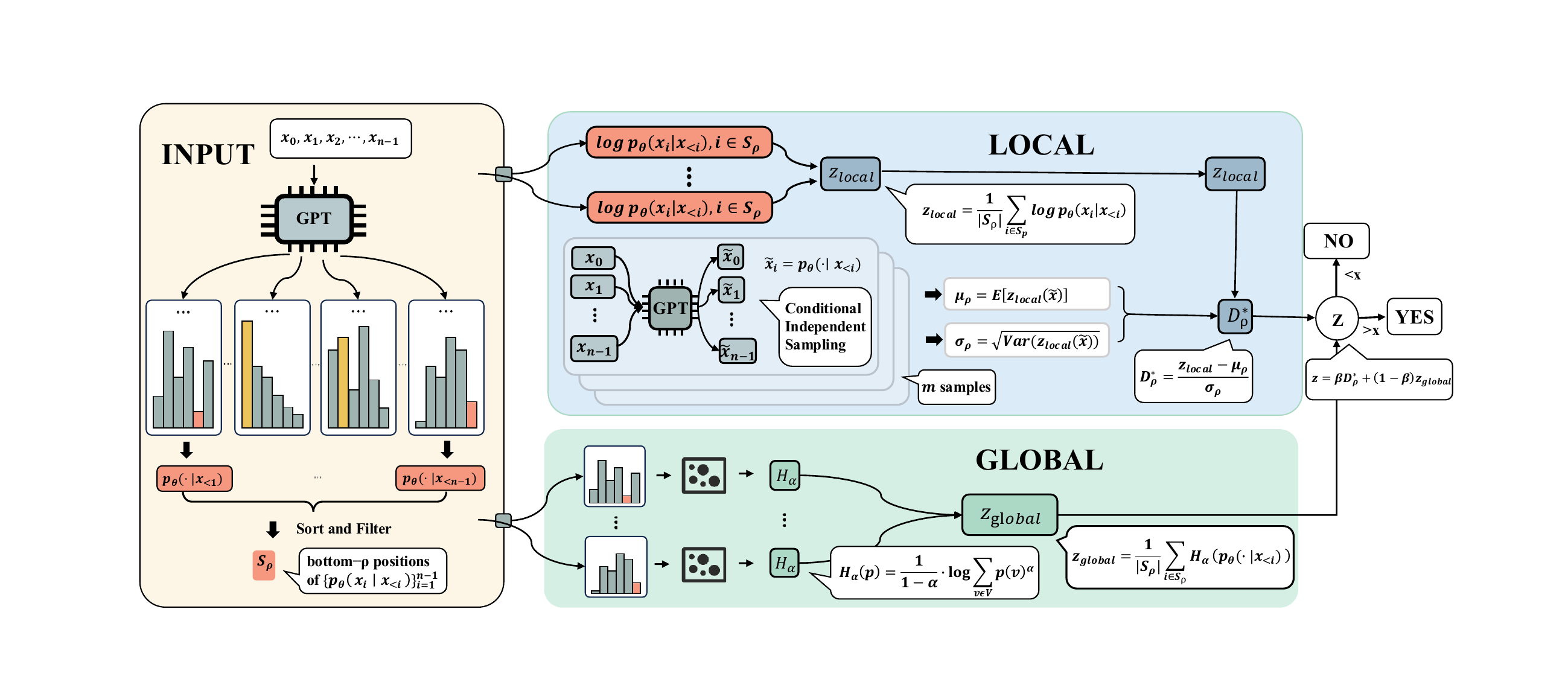}
\caption{Overview of the proposed multiscale uncertainty framework.
The method exploits the empirical observation that low-probability tokens carry stronger discriminative signals between AI-generated and human-written text.
We compute a percentile-based local uncertainty by aggregating log-probabilities over the bottom-$\rho$ tokens, and a global uncertainty using R\'enyi entropy to characterize the shape of the conditional distribution.
These two complementary signals are combined into a unified detection score.
The figure depicts \textsc{Uncertainty++}, which normalizes the local signal $z_{\text{local}}$ via conditional independent sampling before fusion with $z_{\text{global}}$; the simpler variant \textsc{Uncertainty} is obtained by directly aggregating $z_{\text{local}}$ and $z_{\text{global}}$ without the sampling-based normalization step.}
\label{fig:method}
\end{figure*}

An overview of our framework is shown in Figure~\ref{fig:method}.
To address boilerplate dominance and brittle point estimates, we propose a multiscale uncertainty framework that captures both local and global signals.
First, we hypothesize and empirically verify that low-probability tokens carry stronger discriminative signals, then develop a \emph{percentile-based local uncertainty} measure with justification for its effectiveness (Section~\ref{sec:local}).
Second, motivated by theoretical analysis showing that entropy-based measures offer greater robustness, we introduce an \emph{entropy-based global uncertainty} measure that captures distributional shape beyond point estimates (Section~\ref{sec:global}).
These complementary signals are integrated into a unified detection framework (Section~\ref{sec:unified}). 

\subsection{Percentile-Based Local Uncertainty}
\label{sec:local}

This section presents our percentile-based local uncertainty measure to mitigate boilerplate dominance.
We first formalize the core hypothesis, then provide empirical evidence, define the detection statistic and its normalized variant, and give justification for its effectiveness.

\paragraph{Core Hypothesis and Empirical Evidence.}
\begin{assumption}[Low-Probability Discriminability]
\label{assump:low_prob}
The discriminative signals between AI-generated and human-written text are substantially larger at low-probability tokens than that observed at high-probability tokens. 
\end{assumption}
Our method builds on Assumption~\ref{assump:low_prob}.
On \textit{XSum} with source model LLaMA3-8B, proxy model GPT-J, and percentile level $\rho = 0.15$, the Human--AI LogRank gap at low-probability positions is $\mathbf{1.59}$, versus $\mathbf{0.45}$ at high-probability positions ($3.5\times$). The probability ratio (AI/Human) is $\mathbf{9.69\times}$ at low-probability positions versus $\mathbf{1.38\times}$ at high-probability positions. Full statistics are reported in Appendix~\ref{sec:appendix_empirical}.

\paragraph{Local Uncertainty Statistic.}
Motivated by Assumption~\ref{assump:low_prob}, we adopt a percentile-based strategy that selectively aggregates low-probability tokens.
For a sequence of values $\{y_i\}_{i=1}^{n}$, we define the percentile aggregation operator $\mathcal{Q}_\rho$ at level $\rho \in (0,1]$ as the mean of the bottom $\rho$ fraction:
\begin{equation}
\mathcal{Q}_\rho(\{y_i\}) = \frac{1}{|\mathcal S_\rho|}\sum_{i \in \mathcal{S}_\rho} y_i,
\label{eq:def_Qrho}
\end{equation}
where $\mathcal{S}_\rho$ is the index set of the bottom $\rho$ fraction of the elements. 
Applying this operator to token-level log-probabilities yields the local uncertainty statistic:
\begin{equation}
\label{eq:z_local}
z_{\text{local}} = \mathcal{Q}_\rho\bigl( \{\log p_\theta(x_i \mid x_{<i})\}_{i=1}^{n-1} \bigr).
\end{equation}
The choice of $\rho$ reflects a tradeoff: smaller $\rho$ concentrates on the most discriminative tokens but increases variance, while larger $\rho$ reduces variance but reduces the signal.
We analyze this tradeoff empirically in Section~4 (Figure~\ref{fig:hyperparameter}).

\paragraph{Relative Normalization via Conditional Sampling.}
Although $z_{\text{local}}$ in Equation~\eqref{eq:z_local} is discriminative, its magnitude depends on text length and domain.
To obtain a more robust statistic, we normalize it by comparing the observed value with its expectation under the model’s own conditional distribution, extending Fast-DetectGPT~\cite{bao2024fastdetectgpt} to the setting of low-probability aggregation.

Let $\tilde{x}_i \sim p_\theta(\cdot \mid x_{<i})$ be independently sampled at each position, with each draw conditioned on the original prefix $x_{<i}$ rather than on previously sampled tokens.
We define the \emph{Percentile Discrepancy} at level $\rho$ as:
\begin{equation}
\label{eq:D_rho}
D_\rho = z_{\text{local}} - \mathbb{E}\,\mathcal{Q}_\rho\bigl( \{\log p_\theta(\tilde{x}_i \mid x_{<i})\}_{i=1}^{n-1} \bigr),
\end{equation}
and its normalized form:
\begin{equation}
D_\rho^* = \frac{D_\rho}{\sqrt{\mathrm{Var}\,\mathcal{Q}_\rho\bigl( \{\log p_\theta(\tilde{x}_i \mid x_{<i})\}_{i=1}^{n-1} \bigr)}}.
\label{eq:D_rho_norm}
\end{equation}
In practice, $D_\rho^*$ is estimated via Monte Carlo sampling with $m$ conditionally independent samples per position, from which we compute the empirical mean and variance.
When $\rho = 1$, we recover the full-sequence discrepancy.
Intuitively, AI-generated text has log-probability significantly higher than its distributional expectation, yielding large positive values of $D_\rho^*$;
human-written text, in contrast, has log-probability close to the expectation, resulting in $D_\rho^* \approx 0$.
This is verified empirically in Appendix~\ref{sec:appendix_jensen_gap}.

\paragraph{Jensen Discrepancy.}

While the intuition that AI-generated text yields positive discrepancy holds for any $\rho$, focusing on low-probability tokens introduces an additional effect related to the concavity of the percentile operator.

\begin{proposition}[Concavity of the Percentile Operator]
\label{prop:concavity}
The operator $\mathcal{Q}_\rho$ defined in Equation \eqref{eq:def_Qrho} is concave.
Hence, for random variables $\{Y_i\}_{i=1}^{n-1}$, Jensen's inequality yields
\begin{equation*}
\mathbb{E}\,\mathcal{Q}_\rho(\{Y_i\}_{i=1}^{n-1}) \leq \mathcal{Q}_\rho(\{\mathbb{E}\,Y_i\}_{i=1}^{n-1}). 
\end{equation*}
\end{proposition}

\begin{proof}
See Appendix~\ref{sec:appendix_concavity_proof}.
\end{proof}

Based on Proposition~\ref{prop:concavity}, we define the \emph{Jensen discrepancy}
\begin{equation}
\widetilde{D}_\rho = z_{\text{local}} - \mathcal{Q}_\rho\bigl( \{\mathbb{E}\log p_\theta(\tilde{x}_i \mid x_{<i})\}_{i=1}^{n-1} \bigr)
\end{equation}
as a lower bound of $D_\rho$. 
Let $\widetilde{D}_\rho^*$ be its normalized form.
Next, we empirically examine the relationship among $D_\rho$, $\widetilde{D}_\rho$, and $D_1$ with their normalized counterparts.

\begin{proposition}[Empirical Asymmetric Jensen Amplification]
\label{prop:asymmetric_gap}
Across multiple datasets and source models (see Appendix~\ref{sec:appendix_jensen_gap}), we observe that for AI-generated text, $D_\rho^* > \widetilde{D}_\rho^*$, while $\widetilde{D}_\rho^*$ remains comparable to the full-sequence discrepancy $D_1^*$; while for human-written text, $D_\rho^* \approx \widetilde{D}_\rho^* \approx D_1^* \approx 0$.
\end{proposition}

Empirical results summarized in Proposition \ref{prop:asymmetric_gap} shows that the Jensen amplification effect is pronounced only for AI-generated text, reflecting the larger dispersion of token-level log-probabilities at low-probability positions, whereas such effects are negligible for human-written text.
Across 3 datasets, 12 source models, and $5{,}400$ paired samples, for AI-generated text the normalized discrepancy $D_\rho^* = \mathbf{2.18}$ exceeds both the Jensen lower bound $\widetilde{D}_\rho^* = \mathbf{1.34}$ and the full-sequence baseline $D_1^* = \mathbf{1.43}$. For human-written text, all three remain near zero ($D_\rho^* \approx \widetilde{D}_\rho^* \approx D_1^* \approx 0$).
Together, these results explain why that percentile selection combined with Jensen amplification yields better separation.

\subsection{Entropy-Based Global Uncertainty}
\label{sec:global}

This section introduces our global uncertainty measure based on R\'enyi entropy.
While local uncertainty captures how surprising the observed token choice is, global uncertainty characterizes the shape of the entire conditional distribution, thereby providing complementary information that is more robust to surface-level perturbations.

\paragraph{Robustness of Entropy-Based Measures.}
Point estimates such as log-probability are highly sensitive to small perturbations in the input text, which can lead to large fluctuations in detection scores. 
In contrast, entropy-based measures characterize the overall conditional probability distribution rather than relying on a single realized token, and therefore is expected to be more robust to such variations.
To formalize this intuition, we introduce a \emph{multiplicative perturbation model} for AI-generated text in which the conditional probability of a token $x_i$ is reduced by a factor $\gamma > 1$:
\begin{equation}
\label{eq:mult_perturbation_model}
p_\theta(x_i \mid x_{<i}) \rightarrow \frac{1}{\gamma} p_\theta(x_i \mid x_{<i}).
\end{equation}
Such perturbations arise from paraphrasing or changes in decoding strategies, where a high-probability token is replaced by a semantically similar but lower-probability alternative.

\begin{proposition}[Robustness of R\'enyi Entropy at Low-Probability Positions]
\label{prop:entropy_robust}
Consider the multiplicative perturbation model~\eqref{eq:mult_perturbation_model} with proportional redistribution of the remaining probability mass, and let the R\'enyi power sum be $S_\alpha = \sum_{v \in \mathcal{V}} p(v)^\alpha$. For any R\'enyi order $\alpha > 0$ ($\alpha \neq 1$) and any token $x_i$ satisfying $p_\theta(x_i \mid x_{<i}) \leq \tau$ with $\tau \leq (S_\alpha/2)^{1/\alpha}$, the log-probability changes by exactly $\log\gamma$, whereas the change of R\'enyi entropy satisfies
\begin{equation}
\label{eq:renyi_bound}
\begin{aligned}
& |H_\alpha(p') - H_\alpha(p)| \\
& \leq \frac{1}{|1-\alpha|} \max\Bigl( \frac{\alpha\,\tau\,(1-\gamma^{-1})}{1-\tau}, \frac{2\,\tau^{\alpha}(1-\gamma^{-\alpha})}{S_\alpha}\Bigr).
\end{aligned}
\end{equation}
\end{proposition}

\begin{proof}
The proof is provided in Appendix~\ref{sec:entropy_proof}.
\end{proof}

First, the regularity condition $\tau \leq (S_\alpha/2)^{1/\alpha}$ in Proposition~\ref{prop:entropy_robust} is satisfied by low-probability tokens empirically, since their observed probability is typically several orders of magnitude smaller than the threshold for all $\alpha$ used in experiments.
Proposition~\ref{prop:entropy_robust} establishes a sharp contrast between log-probability and R\'enyi entropy under multiplicative perturbations.
Specifically, the change in R\'enyi entropy is small and bounded by $O(\tau^{\min(\alpha,1)})$.
Furthermore, as the perturbation strength $\gamma \rightarrow \infty$, the change in log-probability grows monotonically to infinity, whereas the change in R\'enyi entropy remains bounded.
Consequently, R\'enyi entropy provides a stable uncertainty signal at low-probability positions.

Motivated by this robustness property at the low-probability tokens, R\'enyi entropy serves as a perturbation-resistant statistic, which motivates its integration into our framework.

\paragraph{Global Uncertainty Statistic.}
Based on the above analysis, we characterize global uncertainty using R\'enyi entropy.
The R\'enyi entropy of order $\alpha > 0$ for a discrete distribution $p$ over vocabulary $\mathcal{V}$ is defined as:
\begin{equation*}
H_\alpha(p) = \frac{1}{1-\alpha} \log \sum_{v \in \mathcal{V}} p(v)^\alpha.
\end{equation*}
R\'enyi entropy generalizes Shannon entropy (recovered as $\alpha \to 1$) and provides tunable sensitivity to different regions of the probability mass, where $\alpha < 1$ emphasizes the tail of the distribution and $\alpha > 1$ emphasizes the common tokens.
The optimal $\alpha$ is determined empirically  in Section~4 (Figure~\ref{fig:hyperparameter}).
We define the global uncertainty statistic as the average R\'enyi entropy over low-probability tokens $\mathcal{S}_\rho$:
\begin{equation}
\label{eq:z_global}
z_{\text{global}} = \frac{1}{|\mathcal{S}_\rho|} \sum_{i \in \mathcal{S}_\rho} H_\alpha\bigl(p_\theta(\cdot \mid x_{<i})\bigr).
\end{equation}

\subsection{Unified Framework}
\label{sec:unified}

We integrate local and global uncertainty into a unified detection framework.

\paragraph{Uncertainty.}
We define an uncertainty score by combining local and global uncertainty statistics without sampling:
\begin{equation*}
z_{\textsc{Uncertainty}} = \beta \cdot z_{\text{local}} + (1 - \beta) \cdot z_{\text{global}},
\end{equation*}
where $z_{\text{local}}$ and $z_{\text{global}}$ are defined in Equations~\eqref{eq:z_local} and~\eqref{eq:z_global}, respectively, and $\beta \in [0,1]$ controls their relative contributions.
The sensitivity to the weighting parameter $\beta$ is analyzed in Section~\ref{sec:exp} (Figure~\ref{fig:hyperparameter}).

\paragraph{Uncertainty++.}
We further define a normalized uncertainty score to further improve stability. 
\begin{equation*}
z_{\textsc{Uncertainty++}} = \beta \cdot D_\rho^* + (1 - \beta) \cdot z_{\text{global}},
\end{equation*}
where $D_\rho^*$ is defined in Equation~\eqref{eq:D_rho_norm}.
This normalization calibrates local uncertainty against the model’s distributional expectations, thereby improving robustness to perturbations.


\section{Experiments}
\label{sec:exp}

\begin{table*}[t]
\centering
\small
\setlength{\tabcolsep}{2pt}
\renewcommand{\arraystretch}{0.95}
\captionof{table}{Black-box detection performance (AUROC) across 12 source models. Following Lastde, the AUROC for each model is averaged over three datasets: \textit{XSum}, \textit{WritingPrompts}, and \textit{Reddit}. The \textit{(Diff)} rows report the absolute improvement of our method over the strongest baseline within the same group. Bold and underlined numbers denote the best and second-best results within each group, respectively.}
\label{tab:main_black}
\resizebox{0.959\textwidth}{!}{%
\begin{tabular}{@{}lccccccccccccc@{}}
    \toprule
    \textbf{Method/Models} & \textbf{GPT-2} & \textbf{Neo-2.7} & \textbf{OPT-2.7} & \textbf{Llama-13} & \textbf{Llama2-13} & \textbf{Llama3-8} & \textbf{OPT-13} & \textbf{Bloom-7.1} & \textbf{Falcon-7} & \textbf{Gemma-7} & \textbf{Phi2-2.7} & \textbf{GPT-4-T} & \textbf{Avg.} \\ \midrule
    
    \rowcolor{gray!20} \multicolumn{14}{c}{\textbf{Probability-based Methods}} \\ \midrule
    Likelihood & 66.02 & 67.09 & 67.40 & 65.74 & 68.61 & 99.57 & 68.81 & 61.81 & 67.42 & 69.90 & 73.93 & \uline{79.70} & 71.33 \\
    LogRank & 70.21 & 71.20 & 72.34 & 70.27 & 72.67 & \textbf{99.68} & 73.02 & 67.47 & 71.69 & 72.21 & 78.02 & 79.24 & 74.83 \\
    DetectLRR & 78.22 & 79.21 & 81.21 & 78.44 & 78.80 & 96.62 & 80.23 & 79.49 & 79.89 & 73.58 & 83.82 & 73.86 & 80.28 \\
    Lastde & \uline{89.78} & \uline{90.10} & \uline{90.04} & \uline{80.43} & \uline{80.22} & \uline{99.60} & \uline{89.70} & \textbf{89.27} & \uline{84.59} & \uline{79.55} & \uline{88.57} & 75.98 & \uline{86.49} \\
    \textbf{Uncertainty} & \textbf{92.56} & \textbf{93.57} & \textbf{91.89} & \textbf{82.06} & \textbf{83.76} & 98.43 & \textbf{90.54} & \uline{88.92} & \textbf{86.62} & \textbf{83.37} & \textbf{91.34} & \textbf{81.78} & \textbf{88.74} \\
    \textit{(Diff)} & \textit{2.78} & \textit{3.47} & \textit{1.85} & \textit{1.63} & \textit{3.54} & \textit{-1.25} & \textit{0.84} & \textit{-0.35} & \textit{2.03} & \textit{3.82} & \textit{2.77} & \textit{2.08} & \textit{2.25} \\ \midrule

    \rowcolor{gray!20} \multicolumn{14}{c}{\textbf{Sampling-based Methods}} \\ \midrule
    DetectGPT & 67.56 & 69.28 & 72.03 & 66.12 & 67.96 & 82.90 & 73.89 & 61.83 & 68.69 & 66.55 & 72.76 & 81.73 & 70.94 \\
    DetectNPR & 68.07 & 68.41 & 73.06 & 67.83 & 70.60 & 96.75 & 75.13 & 63.00 & 70.42 & 65.72 & 74.08 & 79.94 & 72.75 \\
    DNA-GPT & 64.15 & 62.63 & 63.64 & 60.77 & 66.71 & \textbf{99.47} & 65.75 & 62.01 & 65.08 & 62.59 & 72.02 & 70.75 & 67.97 \\
    Fast-DetectGPT & 89.51 & 88.76 & 86.52 & 77.59 & 77.60 & \uline{99.31} & 86.14 & 84.55 & 81.41 & 81.46 & 86.68 & \uline{88.19} & 85.64 \\
    Lastde++ & \uline{94.70} & \uline{95.47} & \uline{93.97} & \uline{85.26} & \uline{85.65} & 98.93 & \uline{93.49} & \uline{92.23} & \uline{89.60} & \uline{87.59} & \uline{92.88} & \textbf{88.28} & \uline{91.51} \\
    \textbf{Uncertainty++} & \textbf{96.84} & \textbf{97.40} & \textbf{96.63} & \textbf{89.32} & \textbf{89.04} & 97.68 & \textbf{95.41} & \textbf{94.56} & \textbf{92.82} & \textbf{89.26} & \textbf{95.01} & 84.93 & \textbf{93.24} \\
    \textit{(Diff)} & \textit{2.14} & \textit{1.93} & \textit{2.66} & \textit{4.06} & \textit{3.39} & \textit{-1.79} & \textit{1.92} & \textit{2.33} & \textit{3.22} & \textit{1.67} & \textit{2.13} & \textit{-3.35} & \textit{1.73} \\ \bottomrule

  \end{tabular}
}
\end{table*}

\begin{table*}[t]
\centering
\small
\setlength{\tabcolsep}{2pt}
\renewcommand{\arraystretch}{0.95}
\captionof{table}{White-box detection performance (AUROC) across 12 source models. Following Lastde, the AUROC for each model is averaged over three datasets: \textit{XSum}, \textit{SQuAD}, and \textit{WritingPrompts}. The meanings of \textit{(Diff)}, boldface, and underlining are the same as in Table~\ref{tab:main_black}.}
\label{tab:main_white}
\resizebox{0.959\textwidth}{!}{%
    \begin{tabular}{@{}lccccccccccccc@{}}
    \toprule
    \textbf{Method/Models} & \textbf{GPT-2} & \textbf{Neo-2.7} & \textbf{OPT-2.7} & \textbf{GPT-J} & \textbf{Llama-13} & \textbf{Llama2-13} & \textbf{Llama3-8} & \textbf{OPT-13} & \textbf{Bloom-7.1} & \textbf{Falcon-7} & \textbf{Gemma-7} & \textbf{Phi2-2.7} & \textbf{Avg.} \\ \midrule
    
    \rowcolor{gray!20} \multicolumn{14}{c}{\textbf{Probability-based Methods}} \\ \midrule
    Likelihood & 91.66 & 89.40 & 88.08 & 84.96 & 63.66 & 65.36 & 98.34 & 84.45 & 88.01 & 76.77 & 70.09 & 89.67 & 82.54 \\
    LogRank & 94.15 & 92.87 & 90.99 & 88.69 & 68.87 & 70.29 & \uline{99.02} & 87.77 & 92.41 & 81.28 & 74.83 & 92.12 & 86.11 \\
    DetectLRR & 96.16 & 96.07 & 93.13 & 92.25 & 81.45 & 80.98 & 98.39 & 91.00 & 96.33 & 87.39 & 81.34 & 94.11 & 90.72 \\
    Lastde & \uline{98.43} & \uline{98.62} & \uline{98.15} & \uline{97.10} & \uline{88.50} & \uline{88.62} & \textbf{99.68} & \uline{96.63} & \textbf{99.56} & \uline{95.20} & \uline{91.53} & \uline{96.96} & \uline{95.75} \\
    \textbf{Uncertainty} & \textbf{98.99} & \textbf{98.96} & \textbf{98.69} & \textbf{98.70} & \textbf{91.73} & \textbf{92.55} & 98.69 & \textbf{97.95} & \uline{99.15} & \textbf{97.58} & \textbf{96.03} & \textbf{97.95} & \textbf{97.25} \\
    \textit{(Diff)} & \textit{0.56} & \textit{0.34} & \textit{0.54} & \textit{1.60} & \textit{3.23} & \textit{3.93} & \textit{-0.99} & \textit{1.32} & \textit{-0.40} & \textit{2.37} & \textit{4.50} & \textit{0.99} & \textit{1.50} \\ \midrule

    \rowcolor{gray!20} \multicolumn{14}{c}{\textbf{Sampling-based Methods}} \\ \midrule
    DetectGPT & 93.43 & 90.40 & 90.36 & 83.82 & 63.78 & 65.39 & 70.13 & 85.05 & 89.28 & 77.98 & 68.96 & 89.55 & 80.68 \\
    DetectNPR & 95.77 & 94.77 & 93.24 & 88.86 & 68.60 & 69.83 & 95.55 & 89.78 & 94.95 & 83.06 & 74.74 & 93.06 & 86.85 \\
    DNA-GPT & 89.92 & 86.80 & 86.79 & 82.21 & 62.28 & 64.46 & 98.07 & 82.51 & 86.74 & 74.04 & 63.63 & 88.00 & 80.46 \\
    Fast-DetectGPT & 99.54 & 99.49 & \uline{98.78} & \uline{98.95} & 93.45 & 93.38 & \textbf{99.90} & 98.21 & \uline{99.59} & 97.93 & 96.88 & \uline{98.10} & 97.85 \\
    Lastde++ & \textbf{99.72} & \textbf{99.88} & 99.46 & 99.51 & \uline{96.70} & \uline{96.63} & \uline{98.82} & \uline{99.12} & \textbf{99.94} & \uline{99.18} & \uline{98.38} & \textbf{98.75} & \uline{98.93} \\
    \textbf{Uncertainty++} & \uline{99.69} & \uline{99.79} & \textbf{99.67} & \textbf{99.74} & \textbf{97.41} & \textbf{98.01} & 99.50 & \textbf{99.33} & \uline{99.75} & \textbf{99.71} & \textbf{98.91} & 97.26 & \textbf{99.07} \\
    \textit{(Diff)} & \textit{-0.03} & \textit{-0.09} & \textit{0.21} & \textit{0.23} & \textit{0.71} & \textit{1.38} & \textit{-0.40} & \textit{0.21} & \textit{-0.19} & \textit{0.53} & \textit{0.53} & \textit{-1.49} & \textit{0.14} \\ \bottomrule
    \end{tabular}
}
\end{table*}

This section is organized around the experimental setup and following research questions:
\textbf{RQ1}: Do Uncertainty and Uncertainty++ outperform other baselines?
\textbf{RQ2}: How does each component contribute to detection performance?
\textbf{RQ3}: Do Uncertainty and Uncertainty++ generalize well across different domains and recent model families?
\textbf{RQ4}: Are Uncertainty and Uncertainty++ robust to different decoding strategies and paraphrasing attacks?
\textbf{RQ5}: Is our proposed low-probability proposition applicable to other methods?

\subsection{Experimental Setup}

\textbf{Datasets. }We evaluate our methods on seven datasets spanning a wide range of domains. Following the experimental protocol in Lastde~\cite{xu25lastde}, we conduct main experiments on four benchmarks: \textit{XSum}~\cite{narayan18xsum} \textit{SQuAD}~\cite{rajpurkar16squad}, \textit{WritingPrompts}~\cite{fan18writing}, and \textit{Reddit}~\cite{baumgartner20reddit}. To further investigate the generalization, we additionally perform cross-domain evaluations on \textit{FiQA}~\cite{maia18fiqa}, \textit{Wiki-csai}~\cite{guo23hc3}, and \textit{arXiv}~\cite{wang24m4gt}. Each dataset contains 150 human-written examples. For each human-written example, we take the first 30 tokens as the prompt to generate an LLM continuation. Detailed descriptions of datasets and prompt are provided in Appendix~\ref{sec:appendix_dataset}.

\textbf{Source models \& proxy model. }Experiments are evaluated on sixteen source models. For consistent comparisons, we include the thirteen LLMs used in Lastde~\cite{xu25lastde}. To further assess on recent LLMs, we additionally conduct experiments on GPT-5~\cite{openai2025gpt5}, Gemini-3-Flash~\cite{googlecloud2026gemini3}, and GLM-4.5-Flash~\cite{zhipuai2026glm45flash}. All methods use GPT-J~\cite{gpt-j-6b} as the proxy model under the black-box setting. Detailed information of models are provided in Appendix~\ref{sec:appendix_model}.

\textbf{Baselines \& metric. }We select nine representative statistical baselines and categorize them into two groups: \emph{Probability-based} and \emph{Sampling-based}. The key difference is that probability-based methods compute likelihood-related statistics directly from a given sample, whereas sampling-based methods rely on sampling to obtain relative scores across samples. Specifically, the probability-based baselines include Log-Likelihood~\cite{gehrmann19gltr}, LogRank~\cite{gehrmann19gltr}, DetectLRR~\cite{su23detectllm}, and Lastde~\cite{xu25lastde}. The sampling-based baselines include DetectGPT~\cite{mitchell23detectgpt}, DetectNPR~\cite{su23detectllm}, DNA-GPT~\cite{yang24dnagpt}, Fast-DetectGPT~\cite{bao2024fastdetectgpt}, and Lastde++~\cite{xu25lastde}. Detailed descriptions are provided in Appendix~\ref{sec:related_work}. Following Fast-DetectGPT and Lastde, we use the widely adopted AUROC as the primary metric. We also report partial results using TPR@1\%FPR and TPR@5\%FPR, which highlight performance under low false-positive rates.

\subsection{Main Results (RQ1)}

We propose Uncertainty as a probability-based method and compare it against probability-based baselines. We further introduce Uncertainty++ as an enhanced variant and compare it with sampling-based baselines. To facilitate fair comparison, we directly adopt the dataset used in Lastde and follow its black-box and white-box evaluation protocols.

\textbf{Overall results. }Results for the black-box and white-box settings are reported in Table~\ref{tab:main_black} and Table~\ref{tab:main_white}, respectively. In general, Uncertainty and Uncertainty++ achieve the best average AUROC in both settings. In the black-box setting, Uncertainty reaches 88.74\% AUROC, outperforming the strongest probability-based baseline, Lastde, by 2.25 points. In the white-box setting, Uncertainty++ achieves the highest average AUROC of 99.07\%. Notably, our methods deliver consistently competitive performance on both open-source and closed-source models. Additional results are provided in Appendix~\ref{sec:black_main} and Appendix~\ref{sec:white_main}. We also report white-box results under TPR@1\%FPR and TPR@5\%FPR in Appendix~\ref{sec:fpr} to show low false-positive risk performance.

\begin{figure}[htbp]
  \centering
  \includegraphics[width=\linewidth]{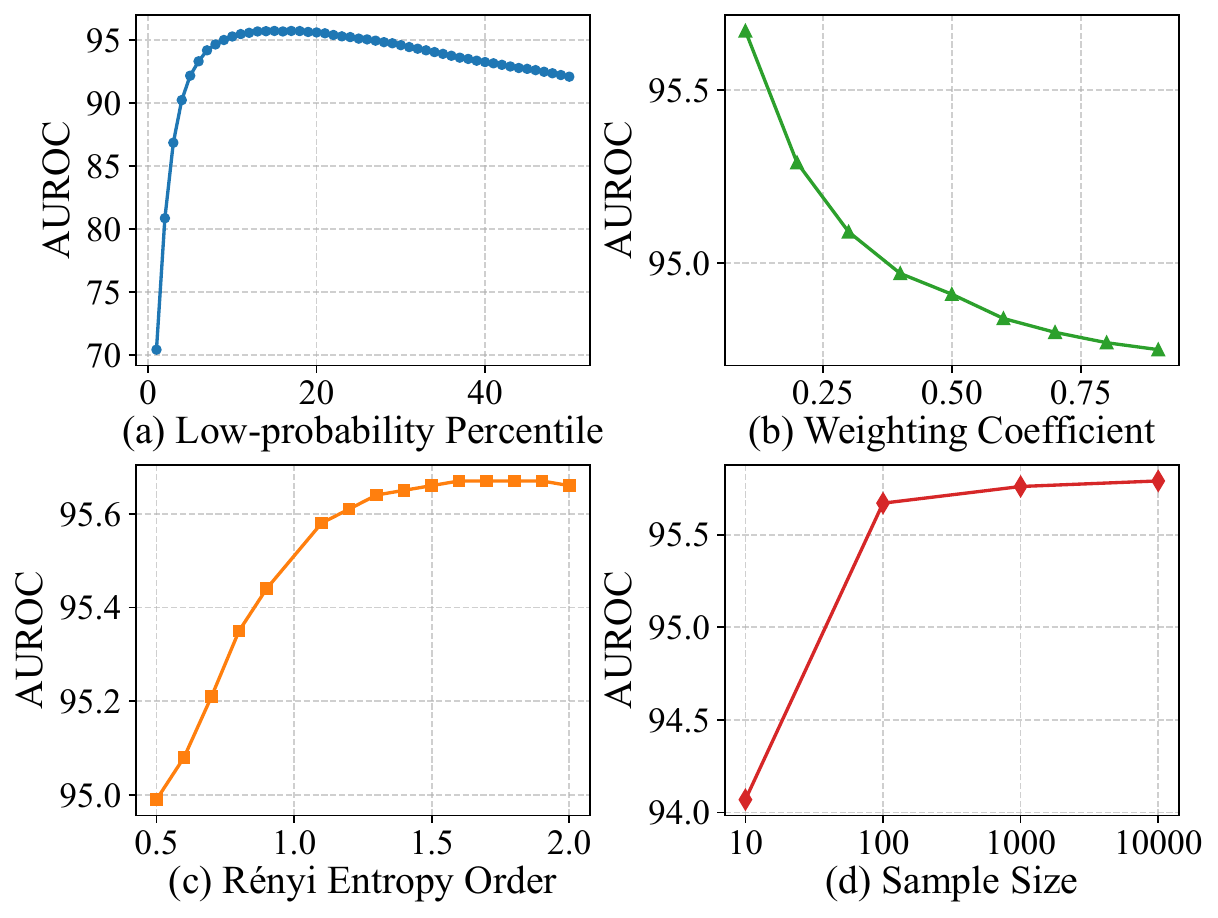}
  \caption{Hyperparameter sensitivity of Uncertainty++ under black-box setting on the \textit{Reddit} dataset. We report AUROC as a function of (a) low-probability percentile $\rho$, (b) weighting coefficient $\beta$, (c) R{\'e}nyi entropy order $\alpha$, and (d) sample size $m$, while keeping other hyperparameters fixed to default values; the boundary cases $\beta=0$ and $\beta=1$ are reported in the ablation study.}
  \label{fig:hyperparameter}
\end{figure}

\textbf{Hyperparameter analysis. }We analyze the hyperparameter sensitivity of Uncertainty++ under the black-box setting from four aspects: (a) low-probability percentile, (b) weighting coefficient, (c) R{\'e}nyi entropy order, and (d) sample size. As shown in Figure~\ref{fig:hyperparameter}, varying the low-probability percentile exhibits a clear unimodal pattern, performance increases initially and then declines as the percentile grows. When the percentile is too small, the selected tail subset contains insufficient signal. The best performance peaks at approximately 15\%, after which AUROC begins to drop. In contrast, varying weighting coefficient and R{\'e}nyi entropy order induces marginal changes (within 1\%), suggesting the insensitivity to these hyperparameters. With respect to the sampling budget, Uncertainty++ nearly converges with 100 samples. 

\begin{table*}[htbp]
\centering
\caption{Ablation study on the \textit{Reddit} dataset under the black-box setting. w/o low-probability removes the low-probability token selection module, w/o logp removes the log-probability (local uncertainty) component, and w/o entropy removes the R{\'e}nyi-entropy (global uncertainty) component. Bold and underlined numbers denote the best and second-best results, respectively.}
\label{tab:ablation}
\resizebox{\textwidth}{!}{
    \setlength{\tabcolsep}{2pt}
    \footnotesize
    \begin{tabular}{@{}lcccccccccccccr@{}}
    \toprule
    \textbf{Methods/Model} & \textbf{GPT-2} & \textbf{Neo-2.7} & \textbf{OPT-2.7} & \textbf{Llama-13} & \textbf{Llama2-13} & \textbf{Llama3-8} & \textbf{OPT-13} & \textbf{Bloom-7.1} & \textbf{Falcon-7} & \textbf{Gemma-7} & \textbf{Phi2-2.7} & \textbf{GPT-4-T} & \textbf{Avg.} & \textbf{(Diff)} \\ \midrule
    
    \rowcolor{gray!20} \multicolumn{15}{c}{\textbf{Reddit}} \\ \midrule
    
    Full & \textbf{98.39} & \textbf{98.93} & \textbf{96.40} & \textbf{92.59} & \textbf{94.31} & 98.12 & \textbf{95.41} & \textbf{97.46} & \textbf{93.15} & \textbf{96.36} & \textbf{97.18} & \uline{88.91} & \textbf{95.60} & \\
    w/o low-probability & 91.57 & 92.32 & 83.44 & 81.09 & 84.68 & \textbf{99.93} & 82.06 & 90.35 & 82.99 & 91.00 & 91.61 & \textbf{95.28} & 88.86 & -6.74 \\
    w/o logp & 55.76 & 58.04 & 68.28 & 69.09 & 69.70 & 60.46 & 66.90 & 62.28 & 69.36 & 59.58 & 68.35 & 57.78 & 63.80 & -31.80 \\
    w/o entropy & \uline{97.72} & \uline{98.39} & \uline{95.12} & \uline{89.84} & \uline{92.60} & \uline{98.56} & \uline{94.70} & \uline{97.41} & \uline{92.12} & \uline{95.90} & \uline{96.98} & 87.20 & \uline{94.71} & -0.89 \\ \bottomrule
    \end{tabular}
}
\end{table*}

\subsection{Ablation Experiments (RQ2)}
\label{sec:ablation}

Table~\ref{tab:ablation} reports the ablation results on \textit{Reddit} in the black-box setting, which confirm that each module in Uncertainty++ contributes to detection. Removing the low-probability selection causes a clear drop from 95.60 to 88.86 AUROC on average (-6.74), showing that focusing on tail tokens is crucial for avoiding boilerplate dominance and retaining discriminative evidence. Disabling the log-probability component yields the largest degradation, reducing the average AUROC to 63.80 (-31.80), which indicates that the local tail log-probability discrepancy is the primary signal driving separation. In contrast, removing the entropy term only slightly reduces performance to 94.71 (-0.89), suggesting that entropy plays a complementary role by stabilizing decisions when probability-based cues are less reliable.

\subsection{Generalization Results (RQ3)}

\begin{table}[htbp]
\centering
\caption{Results across different domains under the black-box setting. The best result is in bold, and the second best underlined.}
\label{tab:domain}
\footnotesize 
\begin{tabularx}{\columnwidth}{@{}lCCCC@{}} 
\toprule
\textbf{Methods/Domain} & \textbf{FiQA} & \textbf{Wiki} & \textbf{arXiv} & \textbf{Avg.} \\ \midrule
\rowcolor{gray!15} \multicolumn{5}{c}{\textbf{Probability-based Methods}} \\ \midrule
Likelihood & \uline{74.41} & 68.47 & 72.97 & 71.95 \\
DetectLRR & 57.55 & 62.62 & 68.55 & 62.91 \\
Lastde & 62.04 & \uline{76.40} & \uline{88.81} & \uline{75.75} \\
\textbf{Uncertainty} & \textbf{79.67} & \textbf{83.89} & \textbf{93.24} & \textbf{85.60} \\ \midrule
\rowcolor{gray!15} \multicolumn{5}{c}{\textbf{Sampling-based Methods}} \\ \midrule
Fast-DetectGPT & 62.13 & 81.42 & 81.72 & 75.09 \\
Lastde++ & \uline{71.97} & \uline{85.58} & \uline{90.95} & \uline{82.83} \\
\textbf{Uncertainty++} & \textbf{77.36} & \textbf{87.90} & \textbf{94.08} & \textbf{86.45} \\ \bottomrule
\end{tabularx}
\end{table}

To further evaluate the generalization ability, we construct datasets that cover three diverse domains, six response lengths, and three recent LLMs, and compare our methods against strong baselines under the black-box setting.

\textbf{Different domains. }To further evaluate cross-domain generalization, we conduct additional experiments on \textit{FiQA}, \textit{Wiki-csai}, and \textit{arXiv}. We use GPT-5 as the source model, and the prompt is identical to the main experiments.

As shown in Table~\ref{tab:domain}, Uncertainty and Uncertainty++ achieve the best performance on all three domains, improving over the strongest baselines by 9.85 and 3.62 AUROC points on average, respectively. Nevertheless, we observe a clear domain gap: performance on \textit{FiQA} remains below 80\% AUROC for all methods, indicating that detecting AI-generated text in highly specialized financial content is still challenging and may require domain-aware modeling.

\begin{figure}[htbp]
    \centering
    \includegraphics[width=\linewidth]{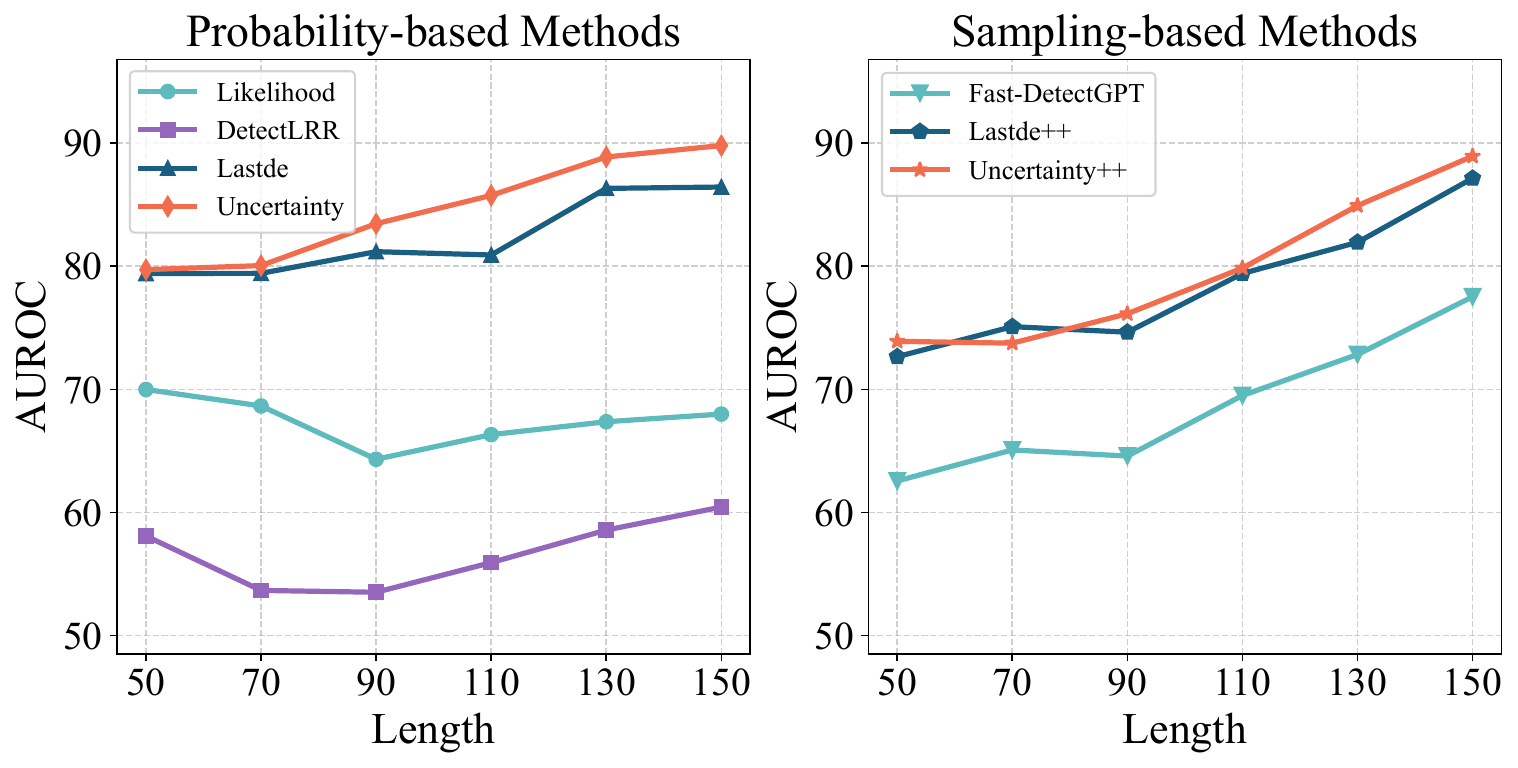}
    \caption{Results across different lengths under black-box setting. The x-axis denotes word number, and the y-axis denotes AUROC.}
    \label{fig:length}
\end{figure}

\textbf{Response lengths. }Detecting shorter texts is generally more challenging due to the limited information~\cite{tian2024multiscale}. To quantify this effect, we conduct experiments on the \textit{arXiv} dataset, where both human-written and AI-generated text are truncated to \{50, 70, 90, 110, 130, 150\} words.

As shown in Figure~\ref{fig:length}, Uncertainty and Uncertainty++ achieve the best average performance within their respective groups, outperforming the strongest baseline by 2.33 and 1.10 AUROC points on average, respectively.

\textbf{Recent LLMs. }To further assess the performance of Uncertainty and Uncertainty++ on recent LLMs, we select GPT-5, Gemini-3-Flash, and GLM-4.5-Flash as source models and conduct experiments on the \textit{arXiv} dataset under black-box setting. The prompt is
identical to the main experiments.

\begin{table}[htbp]
\centering
\caption{Experimental results across recent LLMs under black-box setting. The best result is in bold and the second best underlined.}
\label{tab:model}
\setlength{\tabcolsep}{3pt} 
\footnotesize 
\begin{tabularx}{\columnwidth}{@{}lCCCC@{}}
\toprule
\textbf{Methods/Model} & \textbf{GPT-5} & \textbf{Gemini-3} & \textbf{GLM-4.5} & \textbf{Avg.} \\ \midrule
\rowcolor{gray!15} \multicolumn{5}{c}{\textbf{Probability-based Methods}} \\ \midrule
Likelihood     & 72.97 & 78.13 & 74.72 & 75.27 \\
DetectLRR      & 68.55 & 70.90 & 72.47 & 70.64 \\
Lastde         & \uline{88.81} & \uline{88.29} & \uline{95.34} & \uline{90.81} \\
\textbf{Uncertainty} & \textbf{93.24} & \textbf{90.30} & \textbf{95.41} & \textbf{92.98} \\ \midrule
\rowcolor{gray!15} \multicolumn{5}{c}{\textbf{Sampling-based Methods}} \\ \midrule
Fast-DetectGPT & 81.72 & 91.44 & 92.02 & 88.39 \\
Lastde++       & \uline{90.95} & \uline{92.12} & \uline{95.54} & \uline{92.87} \\
\textbf{Uncertainty++} & \textbf{94.08} & \textbf{92.84} & \textbf{97.46} & \textbf{94.79} \\ \bottomrule
\end{tabularx}
\end{table}

As shown in Table~\ref{tab:model}, Uncertainty and Uncertainty++ achieve the best average performance within their respective groups, outperforming the strongest baseline by 2.17 and 1.92 AUROC points on average, respectively. Notably, our methods maintain consistently strong performance across all evaluated source models (AUROC\textgreater 90\%), indicating robust generalization to text generated by recent LLMs.

\subsection{Robustness Results (RQ4)}
To further assess robustness, we evaluate Uncertainty and Uncertainty++ under a range of decoding variations and LLM-based paraphrasing attacks in the black-box setting.

\begin{figure}[t]
  \centering
  \includegraphics[width=\columnwidth]{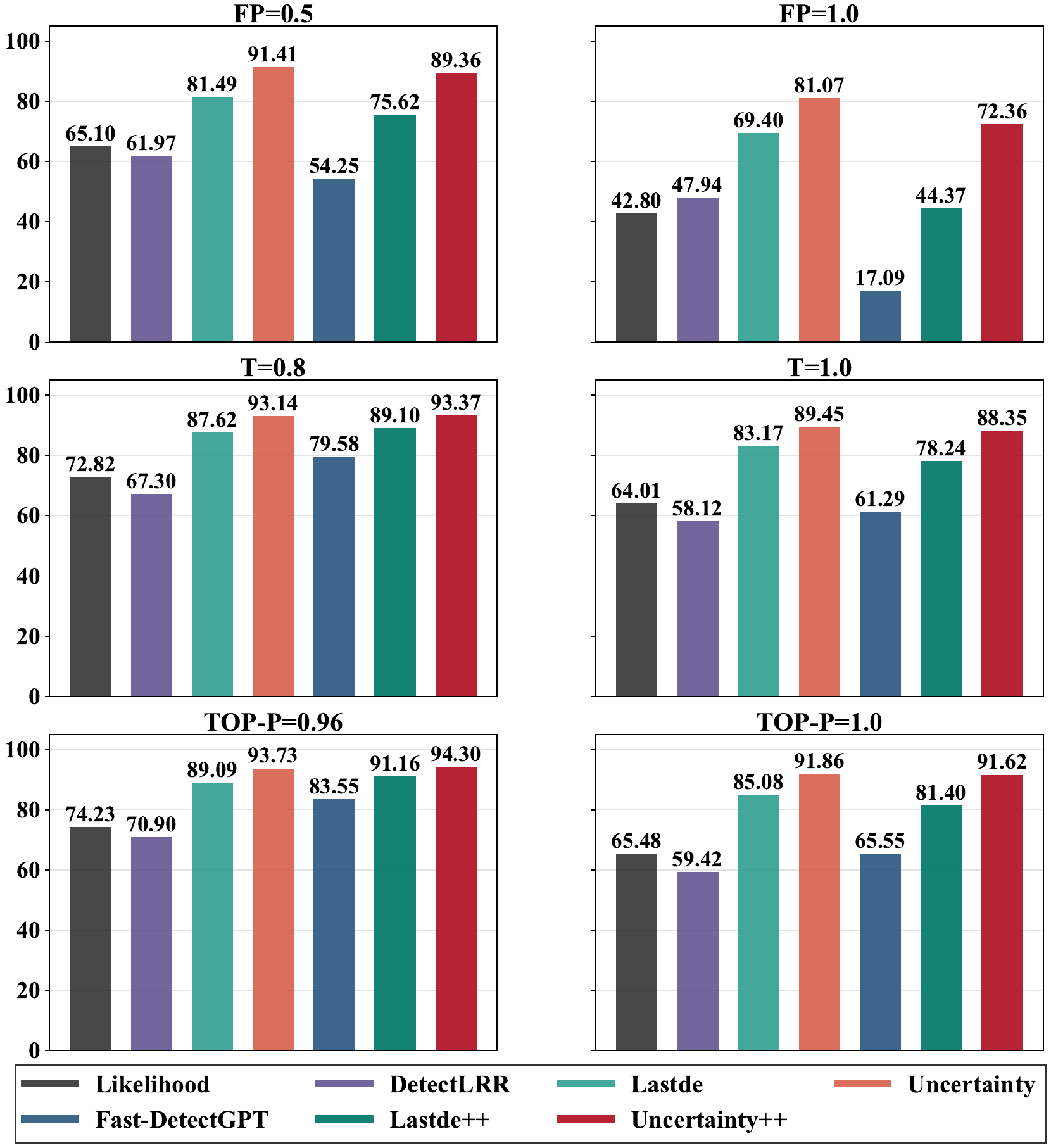}
  
  \caption{Results under different decoding strategies in the black-box setting on the \textit{arXiv} dataset, with GPT-5 as the source model.}
  \label{fig:decoding}
\end{figure}

\textbf{Decoding strategies. }Following Lastde, we examine robustness under different decoding strategies. We adopt GPT-5 as the source model and use \textit{arXiv} as the evaluation dataset. In addition to varying temperature T and top-$p$, we include frequency penalty (FP), which encourages diverse generations by down-weighting tokens that have already appeared.

As shown in Figure~\ref{fig:decoding}, Uncertainty and Uncertainty++ achieve the best average performance within their respective groups, outperforming Lastde and Lastde++ by 7.47 and 11.58 AUROC points on average, respectively. Notably, the baselines exhibit noticeable degradation when FP increases, whereas Uncertainty and Uncertainty++ remain relatively stable across decoding settings. These results suggest that leveraging distributional-shape information, rather than relying solely on token-level probabilities (e.g., Likelihood) or sampling-based signals (e.g., Fast-DetectGPT), improves robustness to decoding-induced distribution shifts.

\begin{table}[htbp]
\centering

\caption{Experimental results across different paraphrasing models under the black-box setting. Original denotes detection results on the raw AI-generated texts. GPT-5 and Phi-2 denote detection results on texts paraphrased by the corresponding LLMs. Avg. is the average performance over paraphrased texts. The best result is in bold, and the second best result is underlined.}

\label{tab:paraphrase}
\setlength{\tabcolsep}{2.5pt} 
\footnotesize 
\begin{tabularx}{\columnwidth}{@{}lCCCC@{}}
\toprule
\textbf{Methods/Strategies} & \textbf{\mbox{Original}} & \textbf{\mbox{GPT-5}} & \textbf{Phi-2} & \textbf{Avg.} \\ \midrule
\rowcolor{gray!15} \multicolumn{5}{c}{\textbf{Probability-based Methods}} \\ \midrule
Likelihood      & 72.97 & 62.28 & 48.59 & 55.44 \\
DetectLRR       & 68.55 & 49.06 & 52.19 & 50.63 \\
Lastde          & \uline{88.81} & \uline{77.15} & \textbf{85.76} & \uline{81.46} \\
\textbf{Uncertainty} & \textbf{93.24} & \textbf{81.97} & \uline{82.80} & \textbf{82.39} \\ \midrule
\rowcolor{gray!15} \multicolumn{5}{c}{\textbf{Sampling-based Methods}} \\ \midrule
Fast-DetectGPT  & 81.72 & 64.49 & 72.13 & 68.31 \\
Lastde++        & \uline{90.95} & \uline{75.21} & \uline{82.75} & \uline{78.98} \\
\textbf{Uncertainty++} & \textbf{94.08} & \textbf{77.57} & \textbf{85.73} & \textbf{81.65} \\ \bottomrule
\end{tabularx}
\end{table}

\textbf{Paraphrasing attack. }Prior studies~\cite{Hu2023radar} have shown that paraphrasing can substantially reduce the accuracy of AI-generated text detectors. To broaden the evaluation, we adopt GPT-5 and Phi-2 as paraphrasing models. All experiments are conducted on the \textit{arXiv} dataset, where the original AI-generated texts are produced by GPT-5. The paraphrasing prompt template is provided in Appendix~\ref{sec:appendix_dataset}.

As shown in Table~\ref{tab:paraphrase}, our methods achieve the best average performance under paraphrasing attacks. Specifically, Uncertainty attains an average AUROC of 82.39 points, outperforming Lastde by 0.93 points, while Uncertainty++ reaches 81.65 points, exceeding Lastde++ by 2.67 points. However, under stronger paraphrasers like GPT-5, the AUROC of all methods drops by more than 10 points relative to Original, indicating robustness to high-quality paraphrasing remains challenging and will be a key focus of future work.

\subsection{Low-Probability Results (RQ5)}

To examine whether our low-probability proposition generalizes beyond Uncertainty and Uncertainty++, we apply the same low-probability percentile to three representative detectors: Likelihood, LogRank, and Fast-DetectGPT. For each detector, we keep the method unchanged, except that the corresponding statistics are computed only over the low-probability token positions instead of the full sequence. All experiments follow the black-box setting with Neo-2.7 as the source model on \textit{XSum}, \textit{WritingPrompts}, and \textit{Reddit}.

As shown in Figure~\ref{fig:low_prob}, focusing on low-probability tokens consistently improves all three detectors. Overall, the average AUROC increases by 24.53, 18.97, and 4.20 points for Likelihood, LogRank, and Fast-DetectGPT, respectively, lifting all methods above 90\% AUROC on average. In particular, Likelihood benefits the most from restricting attention to the tail, suggesting that their original statistics are diluted by high-probability boilerplate tokens, while concentrating on low-probability regions yields a more stable separation. 

We further illustrate the difference in score distributions between human-written and AI-generated text when computed over all tokens versus only low-probability tokens in Fig.~\ref{fig:low_prob_2}. Experiments are conducted under the black-box setting on the \textit{Reddit} dataset. Restricting computation to low-probability tokens yields a markedly clearer separation between human and AI text than aggregating over all tokens, indicating substantially stronger discriminative power.


\begin{figure}[t]
    \centering
    \includegraphics[width=\columnwidth]{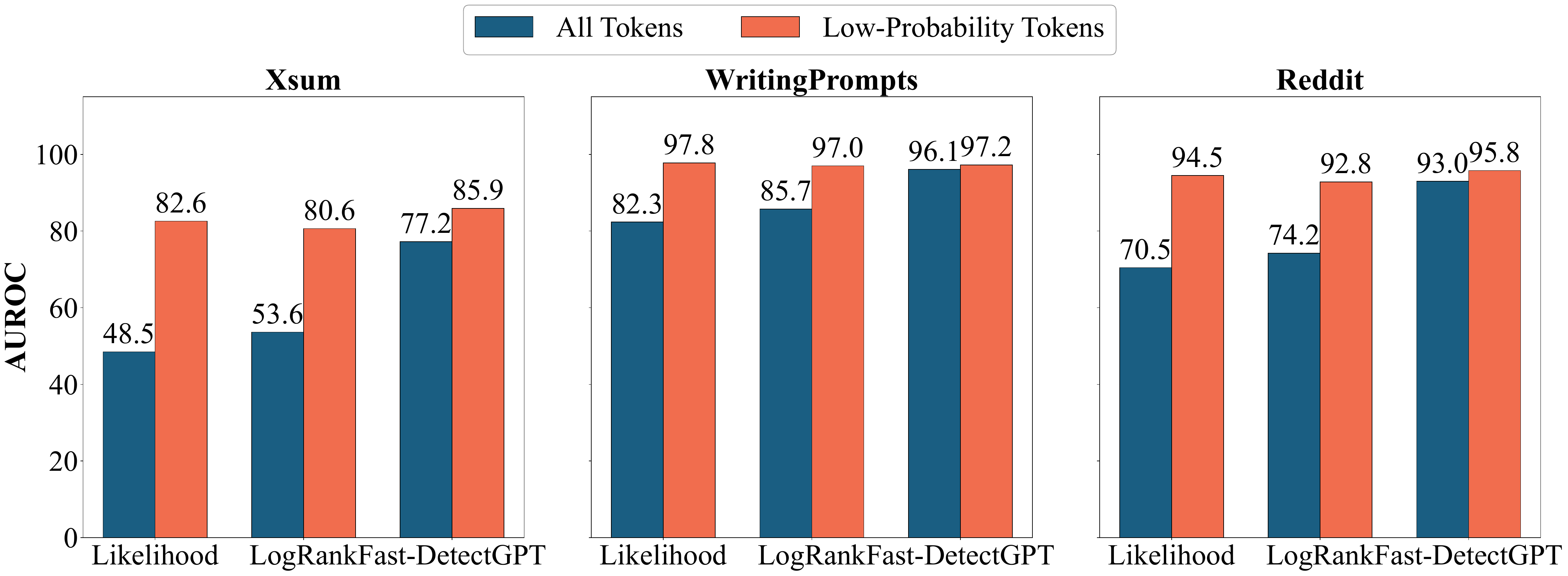}
    \caption{AUROC comparison between using all tokens and low-probability tokens under the black-box setting. Results are reported on \textit{XSum}, \textit{WritingPrompts}, and \textit{Reddit}, with Neo-2.7 as the source model and GPT-J as the proxy model. The low-probability percentile is set to match the setting of main experiments. }
    \label{fig:low_prob}
\end{figure}

\begin{figure}[t]
  \centering

  \begin{subfigure}[t]{0.49\columnwidth}
    \centering
    \includegraphics[width=\linewidth]{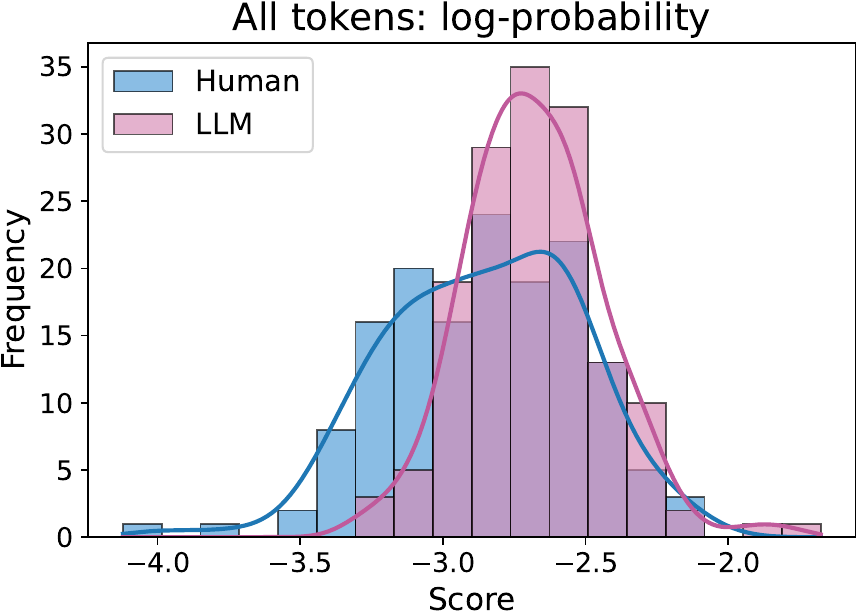}
    \label{fig:all_tokens}
  \end{subfigure}
  \hfill
  \begin{subfigure}[t]{0.49\columnwidth}
    \centering
    \includegraphics[width=\linewidth]{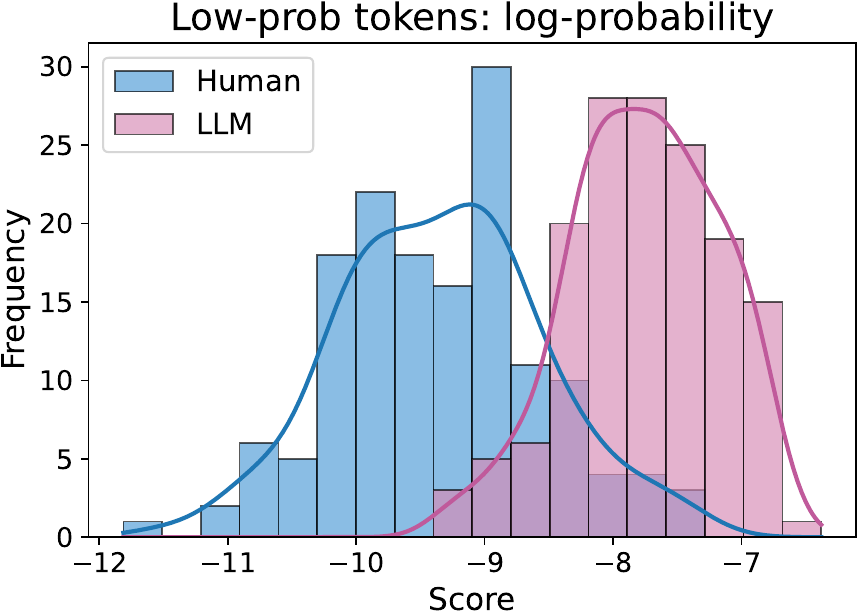}
    \label{fig:low_prob_tokens}
  \end{subfigure}

  \caption{Score distributions for human-written and AI-generated text on the \textit{Reddit} dataset using the Likelihood method under (left) the all-token setting and (right) the low-probability-token setting, with GPT-2 as source model. For clarity, we plot log-probabilities.}
  \label{fig:low_prob_2}
\end{figure}

\section{Conclusion}
This work presents Uncertainty and Uncertainty++, multiscale detectors for zero-shot AI-generated text detection that emphasize low-probability tokens and capture distribution-shape uncertainty via R{\'e}nyi entropy. However, performance can degrade on very short or highly domain-specific texts. In future work, we will further explore two directions: (i) further interpolate uncertainty to mitigate sparse statistical evidence with better calibration; (ii) for domain-specific texts, strengthen domain adaptation of the proxy model to better estimate token probability distributions in specialized domains and enhance uncertainty-based detection. 

\section*{Impact Statement}
This work aims to provide a more reliable statistical signal for identifying AI-generated text, which may help mitigate misinformation, academic misuse, and corpus contamination. At the same time, responsible deployment should explicitly account for false positives and false negatives, report uncertainty when evidence is weak, and avoid treating detector outputs as the sole evidence in high-stakes decisions. In practice, we recommend using our detector as a decision-support tool alongside human review and complementary forensic signals, with thresholds calibrated to the application’s risk tolerance and domain characteristics.

\bibliography{example_paper}
\bibliographystyle{icml2026}

\newpage
\appendix
\onecolumn

\section{Detailed Empirical and Theoretical Analyses}
\label{sec:appendix}

This appendix reports additional theoretical and empirical analyses for the proposed local and global uncertainty measures. 
These analyses justify the design choices and stability properties of the proposed unified uncertainty detection framework.

\subsection{Empirical Validation of Assumption~\ref{assump:low_prob}}
\label{sec:appendix_empirical}

We provide a detailed empirical evidence supporting Assumption~\ref{assump:low_prob}.
We partition tokens into low-probability ($i \in \mathcal{S}_\rho$) and high-probability ($i \notin \mathcal{S}_\rho$) groups with $\rho = 0.15$, and compare human-written and AI-generated text on \textit{XSum}.

\begin{table}[h]
\centering
\caption{Statistical comparison of human-written and AI-generated text on low-probability versus high-probability tokens on \textit{XSum} (source model: LlaMa3-8B, proxy model: GPT-J).
The column $\Delta$ reports the difference in mean LogRank between human and AI text (Human $-$ AI), where larger values indicate greater discriminability.
The column Ratio reports the ratio of mean probability (AI / Human), where values further from 1 indicate greater discriminability.
AI-generated text exhibits lower uncertainty with the gap most pronounced for low-probability tokens.}
\label{tab:low_prob_stats}
\scriptsize
\begin{tabular}{lcccccc}
\toprule
Token Type
& \multicolumn{2}{c}{Mean LogRank}
& $\Delta$
& \multicolumn{2}{c}{Mean Probability}
& Ratio \\
\cmidrule(lr){2-3} \cmidrule(lr){5-6}
& Human & AI & (H $-$ AI) & Human & AI & (AI / H) \\
\midrule
Low-prob.\ ($i \in \mathcal{S}_\rho$, $\rho = 0.15$)
& $4.63$ & $3.04$ & $\mathbf{1.59}$
& $0.0016$ & $0.0155$ & $\mathbf{9.69\times}$ \\
High-prob.\ ($i \notin \mathcal{S}_\rho$)
& $0.71$ & $0.26$ & $0.45$
& $0.380$ & $0.524$ & $1.38\times$ \\
\bottomrule
\end{tabular}
\end{table}

The results reveal a striking pattern.
For low-probability tokens, AI-generated text has substantially lower LogRank than human-written text, with a difference $\Delta = 1.59$.
The probability ratio is even more pronounced: AI-generated tokens are on average $9.69$ times more probable than their human-written counterparts.
In contrast, for high-probability tokens, the LogRank difference shrinks to $\Delta = 0.45$ and the probability ratio drops to $1.38$ times.
This confirms that the discriminative signal is concentrated in low-probability regions, motivating a selective aggregation strategy that focuses on these positions.

\subsection{Proof of Proposition~\ref{prop:concavity}}
\label{sec:appendix_concavity_proof}

\begin{proof}
The operator can be written as $\mathcal{Q}_\rho(\{Y_i\}) = \min_{|S|=\rho n} \frac{1}{|S|}\sum_{i \in S} Y_i$.
Each function $\frac{1}{|S|}\sum_{i \in S} Y_i$ is linear in $\{Y_i\}$, and the pointwise minimum of linear functions is concave.
\end{proof}

\subsection{Empirical Validation of Proposition~\ref{prop:asymmetric_gap}}
\label{sec:appendix_jensen_gap}

We validate Proposition~\ref{prop:asymmetric_gap} on \textit{XSum}, \textit{WritingPrompts}, and \textit{Reddit}, with 12 source models and 5,400 paired samples.

\begin{table}[h]
\centering
\caption{Decomposition of discrepancy statistic. Superscript $*$ denotes normalized variants.}
\scriptsize
\label{tab:percentile_motivation}
\begin{tabular}{lcccccc}
\toprule
\textbf{Source} & $D_\rho$ & $\tilde{D}_\rho$ & $D_1$ & $D_\rho^*$ & $\tilde{D}_\rho^*$ & $D_1^*$ \\
\midrule
\textbf{Human} & 0.04 & 0.04 & $-$0.05 & 0.08 & 0.09 & $-$0.31 \\
\textbf{AI}    & 1.20 & 0.64 & 0.22   & \textbf{2.18} & 1.34 & 1.43 \\
\bottomrule
\end{tabular}
\end{table}

Table~\ref{tab:percentile_motivation} reports the decomposition of discrepancy statistics into the original discrepancy $D_\rho$, its Jensen counterpart $\widetilde{D}_\rho$, and the full-sequence discrepancy $D_1$, together with their normalized variants.

For human-written text, all statistics remain close to zero:
\[
D_\rho^* \approx \widetilde{D}_\rho^* \approx D_1^* \approx 0,
\]
indicating that percentile selection does not induce a noticeable Jensen gap. 
For AI-generated text, a clear separation emerges.
While $\widetilde{D}_\rho^*$ and $D_1^*$ are comparable, $D_\rho^*$ is substantially larger:
\[
D_\rho^* > \widetilde{D}_\rho^* \approx D_1^* > 0.
\]
This demonstrates that combining percentile selection with Jensen amplification yields a stronger discrepancy signal than either component alone.

\newpage

\subsection{Proof of Proposition~\ref{prop:entropy_robust}}
\label{sec:entropy_proof}

\begin{proof}
Let $\pi = p(x_i) = p_\theta(x_i \mid x_{<i}) \leq \tau$ denote the observed token probability.
The multiplicative perturbation with proportional redistribution produces a perturbed distribution $p'$ with
\begin{equation*}
p'(x_i) = \frac{\pi}{\gamma},
\qquad p'(v) = c\,p(v) \text{ for } v \neq x_i,
\qquad c = \frac{1 - \pi/\gamma}{1 - \pi},
\end{equation*}
where $c$ is the scaling factor that ensures the perturbed distribution remains normalized, i.e., $\sum_v p'(v) = 1$.
Since $\gamma \geq 1$, $c$ satisfies $1 \leq c \leq 1/(1-\tau)$.

\textbf{Log-probability.}
Directly, $\bigl|\log p'(x_i) - \log p(x_i)\bigr| = \log\gamma$.

\textbf{Entropy contribution.}
Since $H_\alpha = \frac{1}{1-\alpha}\log S_\alpha$, it suffices to bound $\bigl|\log\bigl(S_\alpha(p')/S_\alpha\bigr)\bigr|$.
Separating the perturbed token from the redistributed mass gives
\begin{equation*}
S_\alpha(p') = (\pi/\gamma)^\alpha + c^\alpha\bigl(S_\alpha - \pi^\alpha\bigr),
\end{equation*}
where $S_\alpha \geq \pi^\alpha$ because $\pi^\alpha$ is one term of the sum.
A direct calculation gives
\begin{equation*}
c - 1 = \frac{\pi(1-\gamma^{-1})}{1-\pi} \geq 0,
\qquad
c - \gamma^{-1} = \frac{\gamma-1}{\gamma(1-\pi)} \geq 0,
\end{equation*}
so $c \geq 1$ and $c \geq 1/\gamma$.
We bound the ratio $S_\alpha(p')/S_\alpha$ from above and below in turn.

\emph{Upper bound.}
Since $\alpha > 0$, the inequality $c \geq 1/\gamma$ yields $(\pi/\gamma)^\alpha \leq c^\alpha \pi^\alpha$, hence
\begin{equation*}
S_\alpha(p') \leq c^\alpha \pi^\alpha + c^\alpha(S_\alpha - \pi^\alpha) = c^\alpha S_\alpha.
\end{equation*}
Applying $\log(1+x) \leq x$ at $x = c-1$, together with monotonicity of $t \mapsto t/(1-t)$ on $[0,1)$ and the assumption $\pi \leq \tau$,
\begin{equation*}
\log\frac{S_\alpha(p')}{S_\alpha}
\leq \alpha\log c
\leq \frac{\alpha\,\pi(1-\gamma^{-1})}{1-\pi}
\leq \frac{\alpha\,\tau(1-\gamma^{-1})}{1-\tau}.
\end{equation*}

\emph{Lower bound.}
Using $c^\alpha \geq 1$ in the decomposition,
\begin{equation*}
S_\alpha(p') \geq (\pi/\gamma)^\alpha + (S_\alpha - \pi^\alpha)
= S_\alpha - \pi^\alpha\bigl(1 - \gamma^{-\alpha}\bigr),
\end{equation*}
and using $\pi^\alpha \leq \tau^\alpha$ (since $\pi \leq \tau$ and $\alpha > 0$),
\begin{equation*}
\frac{S_\alpha(p')}{S_\alpha} \;\geq\; 1 - \frac{\tau^\alpha(1-\gamma^{-\alpha})}{S_\alpha}.
\end{equation*}
The regularity condition $\tau \leq (S_\alpha/2)^{1/\alpha}$ ensures $\tau^\alpha(1-\gamma^{-\alpha})/S_\alpha \in [0, 1/2]$, on which $-\log(1-x) \leq 2x$, whence
\begin{equation*}
-\log\frac{S_\alpha(p')}{S_\alpha} \;\leq\; \frac{2\,\tau^\alpha(1-\gamma^{-\alpha})}{S_\alpha}.
\end{equation*}

\emph{Conclusion.}
Combining the two one-sided bounds,
\begin{equation*}
\left|\log\frac{S_\alpha(p')}{S_\alpha}\right|
\leq \max\!\left(\frac{\alpha\,\tau\,(1-\gamma^{-1})}{1-\tau},\; \frac{2\,\tau^\alpha(1-\gamma^{-\alpha})}{S_\alpha}\right),
\end{equation*}
and dividing by $|1-\alpha|$ yields the claimed bound on $|H_\alpha(p') - H_\alpha(p)|$.
\end{proof}

\newpage

\section{Related Work}
\label{sec:related_work}
We review prior work on AI-generated text detection, covering watermark methods, fine-tuning methods, and statistical methods. For statistical approaches, we further categorize them into probability-based methods and sampling-based methods.

\subsection{Watermark Methods}
\label{sec:appendix_a1_watermark}

Watermarking methods embed a detectable signal at generation time, enabling downstream verification and provenance tracking. For example, Qu et al.~\cite{qu-etal-2025-provably} propose a multi-bit watermark that remains decodable under bounded text edits by distributing watermark evidence across groups of positions and aggregating it for decoding, while Gloaguen et al.~\cite{gloaguen25blackbox} study black-box watermark presence detection: using only a limited number of black-box queries to an API/model, they develop statistical tests to detect whether popular watermarking scheme families are deployed.

\paragraph{Limitations.}
Testing several popular APIs, Gloaguen et al.~\cite{gloaguen25blackbox} find no strong evidence of watermarking. This suggests that, in certain model families, watermark presence may not be detectable.

\subsection{Fine-tuning Methods}

Fine-tuning methods train discriminative models on labeled corpora of human-written and AI-generated text, typically producing a probability that an input is AI-generated~\cite{guo23hc3}). Recent fine-tuning methods further improve robustness by enhancing the training procedure, DP-Net~\cite{zhou25kill} learns dynamic perturbations to create harder training examples and reduce sensitivity to attacks. Learning2Rewrite~\cite{hao-etal-2025-learning} fine-tunes a rewriting model to amplify separability between human and LLM text based on how much rewriting is needed, yielding stable decision.

\paragraph{Limitations.}
Despite strong in-distribution performance, fine-tuned detectors depend on curated labeled data and can still degrade under unseen generators and domains, which limits their suitability for zero-shot detection~\cite{chen25divscore}.

\subsection{Probability-based Methods}
\label{sec:probability}
From a statistical perspective, detectors either (i) use pointwise token scores derived from $p_\theta(\cdot \mid x_{<i})$ (probability-based), or (ii) use sample comparisons to capture relative score beyond the observed token (sampling-based). We first review the probability-based family. These methods capture only \emph{local} uncertainty about the specific token choice.

\paragraph{Likelihood.}
The Likelihood statistic \cite{gehrmann19gltr} computes the average log-probability of the observed tokens:
\begin{equation}
z_{\text{likelihood}}(\mathbf{x})=\frac{1}{n-1}\sum_{i=1}^{n-1}\log p_\theta(x_i\mid x_{<i}).
\end{equation}

\paragraph{LogRank.}
LogRank~\cite{gehrmann19gltr} adopts the rank of observed tokens under the conditional distribution:
\begin{equation}
z_{\text{logrank}}(\mathbf{x})=\frac{1}{n-1}\sum_{i=1}^{n-1}\log r_\theta(x_i\mid x_{<i}),
\end{equation}
where $r_\theta(x_i\mid x_{<i}) \geq 1$ is the rank of $x_i$, determined by ordering all possible tokens at position $i$ in descending order by their probabilities.

\paragraph{DetectLRR.}
LRR \cite{su23detectllm} combines both signals:
\begin{equation}
z_{\text{LRR}}(\mathbf{x})= -\frac{\sum_{i=1}^{n-1}\log p_\theta(x_i\mid x_{<i})}{\sum_{i=1}^{n-1}\log r_\theta(x_i\mid x_{<i})}.
\end{equation}

\paragraph{Lastde.}
Lastde \cite{xu25lastde} introduces a time-series perspective by treating the token log-probability sequence as a temporal signal and extracting local dynamics via diversity entropy. It applies multiscale transformation with scales $\tau = 1, 2, \ldots, \tau'$, where scale-$\tau$ averages consecutive log-probabilities:
\begin{equation}
\log p_\theta^{(\tau)}(j) = \frac{1}{\tau}\sum_{i=j}^{j+\tau-1} \log p_\theta(x_i|x_{<i}).
\end{equation}
For each scale, a sliding window of size $s$ segments the sequence, and cosine similarities between adjacent segments are computed and discretized into $\varepsilon$ bins to form a probability state sequence $P^{(\tau)} = (P_1^{(\tau)}, \ldots, P_\varepsilon^{(\tau)})$. The diversity entropy at scale $\tau$ is as follow:
\begin{equation}
\mathrm{DE}(s, \varepsilon, \tau) = -\frac{1}{\ln \varepsilon}\sum_{i=1}^{\varepsilon} P_i^{(\tau)} \ln P_i^{(\tau)}.
\end{equation}
The final detection score combines log-likelihood with aggregated multiscale diversity entropy:
\begin{equation}
z_{\text{Lastde}}(\mathbf{x}) = \frac{\frac{1}{n-1}\sum_{i=1}^{n-1} \log p_\theta(x_i|x_{<i})}{\mathrm{Agg}\bigl(\mathrm{DE}(s, \varepsilon, 1), \ldots, \mathrm{DE}(s, \varepsilon, \tau')\bigr)},
\end{equation}
where $\mathrm{Agg}(\cdot)$ is an aggregation function (e.g., standard deviation).

\paragraph{Limitations.}
These methods discard the shape of the conditional distribution, retaining only information about the observed token. Although Lastde captures local temporal dynamics of the probability sequence, it still aggregates \emph{uniformly} over all positions, allowing high-probability boilerplate tokens to dominate and dilute the discriminative signal from informative low-probability positions.

\subsection{Sampling-based Methods}
\label{sec:sampling}
These methods address the limitation of raw point estimates by computing \emph{relative} statistics, comparing the observed text's score to those of sampled or perturbed variants.

\paragraph{DetectGPT.}
DetectGPT \cite{mitchell23detectgpt} uses a perturbation model $q_\zeta$ to generate perturbed samples $\tilde{\mathbf{x}} \sim q_\zeta(\cdot|\mathbf{x})$, then computes:
\begin{equation}
z_{\text{DetectGPT}}(\mathbf{x})=\frac{\log p_\theta(\mathbf{x})-\tilde{\mu}}{\tilde{\sigma}},
\end{equation}
where $\tilde{\mu} = \mathbb{E}\bigl[\log p_\theta(\tilde{\mathbf{x}})\bigr]$ and $\tilde{\sigma}^2 = \mathrm{Var}\bigl[\log p_\theta(\tilde{\mathbf{x}})\bigr]$ are the expectation and variance over the perturbed samples. In the experiments, following the setup described in the referenced paper, we used T5-3B as the perturbation model.

\paragraph{DetectNPR.}
Normalized Log-Rank Perturbation (NPR) \cite{su23detectllm} applies the relative comparison idea to ranks. Given perturbed samples $\tilde{\mathbf{x}} \sim q_\zeta(\cdot|\mathbf{x})$ generated by a perturbation model:
\begin{equation}
z_{\text{NPR}}(\mathbf{x})=\frac{\mathbb{E}\bigl[\log r_\theta(\tilde{\mathbf{x}})\bigr]}{\log r_\theta(\mathbf{x})},
\end{equation}
where $\log r_\theta(\mathbf{x}) = \frac{1}{n-1}\sum_{i=1}^{n-1}\log r_\theta(x_i|x_{<i})$ is the averaged log-rank of the original text, and the expectation is over the perturbed samples.

\paragraph{DNA-GPT.}
DNA-GPT \cite{yang24dnagpt} truncates the original text $\mathbf{x}$ to obtain a context $\mathbf{z}$, then samples continuations $\tilde{\mathbf{x}} \sim p_\theta(\cdot|\mathbf{z})$ from the scoring model. The detection score compares the original likelihood with the expected likelihood of sampled continuations:
\begin{equation}
z_{\text{DNA}}(\mathbf{x}) = \log p_\theta(\mathbf{x}) - \mathbb{E}\bigl[\log p_\theta(\tilde{\mathbf{x}})\bigr].
\end{equation}

\paragraph{Fast-DetectGPT.}
Fast-DetectGPT \cite{bao2024fastdetectgpt} replaces expensive perturbation with efficient conditional sampling. At each position $i$, it independently samples $\tilde{x}_i \sim p_\theta(\cdot|x_{<i})$ conditioned on the original prefix, then computes:
\begin{equation}
z_{\text{Fast}}(\mathbf{x})=\frac{\sum_{i=1}^{n-1}\log p_\theta(x_i\mid x_{<i})-\tilde{\mu}}{\tilde{\sigma}},
\end{equation}
where $\tilde{\mu} = \mathbb{E}\bigl[\sum_{i=1}^{n-1}\log p_\theta(\tilde{x}_i\mid x_{<i})\bigr]$ and $\tilde{\sigma}^2 = \mathrm{Var}\bigl[\sum_{i=1}^{n-1}\log p_\theta(\tilde{x}_i\mid x_{<i})\bigr]$ are computed over the conditionally sampled tokens. In the experiments, following the setup described in the referenced paper, we set the number of samples to 10,000.

\paragraph{Lastde++.}
Lastde++ \cite{xu25lastde} extends Lastde by incorporating sampling-based normalization. At each position $i$, it independently samples $\tilde{x}_i \sim p_\theta(\cdot|x_{<i})$ conditioned on the original prefix to form $\tilde{\mathbf{x}}$, then computes:
\begin{equation}
z_{\text{Lastde++}}(\mathbf{x}) = \frac{z_{\text{Lastde}}(\mathbf{x}) - \tilde{\mu}}{\tilde{\sigma}},
\end{equation}
where $\tilde{\mu} = \mathbb{E}\bigl[z_{\text{Lastde}}(\tilde{\mathbf{x}})\bigr]$ and $\tilde{\sigma}^2 = \mathrm{Var}\bigl[z_{\text{Lastde}}(\tilde{\mathbf{x}})\bigr]$ are computed over the conditionally sampled texts.

\paragraph{Limitations.}
While relative normalization improves robustness, these methods still rely on point estimates and aggregate uniformly over all positions, remaining diluted by high-probability boilerplate.

\subsection{Relation to Our Work}

Table~\ref{tab:method_comparison} summarizes the key differences between existing statistical methods and our approach along two dimensions.

\begin{table}[h]
\centering
\caption{Comparison of detection methods along two key dimensions.}
\label{tab:method_comparison}
\begin{tabular}{lcc}
\toprule
\textbf{Method} & \textbf{Uncertainty Type} & \textbf{Aggregation} \\
\midrule
Likelihood, LogRank, LRR & Local (point) & Uniform \\
Lastde & Local (point) + temporal & Uniform \\
DetectGPT, Fast-DetectGPT & Local (point, relative) & Uniform \\
NPR, DNA-GPT & Local (point, relative) & Uniform \\
Lastde++ & Local (point + temporal, relative) & Uniform \\
\midrule
\textbf{Ours} & \textbf{Local + Global} & \textbf{Selective (low-prob.)} \\
\bottomrule
\end{tabular}
\end{table}

Prior methods suffer from one or both of the following limitations:
\begin{itemize}
    \item \textbf{Uniform aggregation}: They treat all token positions equally, allowing high-probability boilerplate to dominate and dilute the discriminative signal.
    \item \textbf{Single-scale uncertainty}: They model uncertainty at a single level (typically token-level local uncertainty), without jointly capturing global distributional uncertainty.
\end{itemize}

Our method addresses both limitations:
\begin{itemize}
    \item \textbf{Selective aggregation}: We focus on low-probability token positions, where the gap between human-written and AI-generated text is most pronounced, avoiding dilution by uninformative boilerplate.
    \item \textbf{Multiscale uncertainty}: We combine local uncertainty (tail log-probability of observed tokens) with global uncertainty (R{\'e}nyi entropy of the conditional distribution), capturing complementary views of the model's uncertainty.
\end{itemize}

This combination of selective aggregation and multiscale uncertainty modeling distinguishes our approach from prior works.

\subsection{Low-probability Signals in Membership Inference}

Membership inference aims to determine whether a given sample is included in a model's training data. Although it differs from AI-generated text detection, both tasks can exploit distributional traces reflected in probability patterns.
Recent studies in membership inference also emphasize low-probability signals. Shi et al.~\cite{shi2024detecting} propose Min-K\% Prob, which detects pre-training data by aggregating the lowest-probability $K\%$ tokens rather than all tokens. Zhang et al.~\cite{zhang2025mink} improve this idea with Min-K\%++, which calibrates low-probability evidence using the conditional distribution over the vocabulary. Raoof et al.~\cite{raoof2025infilling} further extend this line by introducing infilling score, which identifies unlikely tokens using both left and right contexts.

These works support the broader observation that informative signals often concentrate in low-probability regions. Future work may reveal deeper connections between membership inference and AI-generated text detection.

\newpage

\section{Details of Experimental Setup}
\label{sec:experimental setup}
\subsection{Datasets and Prompts}
\label{sec:appendix_dataset}

Following Lastde, we conduct black-box detection on \textit{XSum}, \textit{WritingPrompts}, and \textit{Reddit}, and white-box detection on \textit{XSum}, \textit{SQuAD}, and \textit{WritingPrompts}. We directly adopt the datasets used in Lastde’s main experiments. Each dataset contains 150 paired examples of human-written and AI-generated text, where each AI-generated sample is obtained by continuing from the first 30 tokens of its corresponding human-written text.

For the generalization experiments, we follow the same protocol as Lastde: we randomly sample 150 human-written texts from each of \textit{FiQA}, \textit{Wiki-csai}, and \textit{arXiv}. For the paraphrasing attack, we rewrite the AI-generated texts produced on \textit{arXiv}.

\begin{figure}[t]
  \centering
  \includegraphics[width=0.73\textwidth]{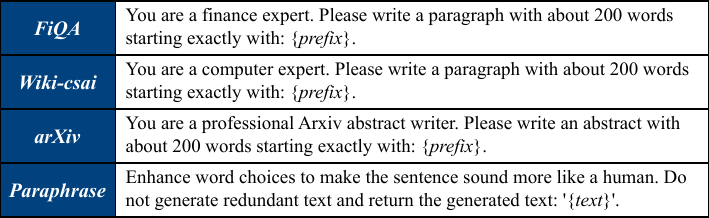}
  \caption{Prompt templates used in our experiments.}
  \label{fig:prompt}
\end{figure}

Figure~\ref{fig:prompt} summarizes the prompts used throughout our experiments. The first three prompts are used to generate model continuations for the domain generalization study on \textit{FiQA}, \textit{Wiki-csai}, and \textit{arXiv}, respectively, where the continuation is required to start exactly with 30 tokens of human-written text. The last prompt is used in the paraphrasing attacks, which rewrites the input AI-generated text to be more human-like.

\subsection{Models}
\label{sec:appendix_model}
Following Lastde, all models used in our experiments are listed in Table~\ref{tab:used_models}. For the black-box setting, all surrogate models are based on GPT-J unless specified.

\begin{table}[t]
\centering
\caption{Models used in our experiments.}
\label{tab:used_models}
\setlength{\tabcolsep}{6pt}
\small
\begin{tabular}{l l r}
\toprule
\textbf{Model} & \textbf{Model File/Service} & \textbf{Parameters} \\
\midrule
GPT-2~\cite{radford2019gpt2}         & openai-community/gpt2-xl       & 1.5B \\
OPT-2.7~\cite{zhang2022opt}          & facebook/opt-2.7b              & 2.7B \\
Neo-2.7~\cite{black2021gptneo}       & EleutherAI/gpt-neo-2.7B        & 2.7B \\
Phi2-2.7~\cite{gunasekar2023phi2}    & microsoft/phi-2                & 2.7B \\
GPT-J~\cite{gpt-j-6b}            & EleutherAI/gpt-j-6B            & 6B \\
Falcon-7~\cite{penedo2023refinedweb} & tiiuae/falcon-7b               & 7B \\
Gemma-7~\cite{gemma2024gemma}        & google/gemma-7b                & 7B \\
BLOOM-7.1~\cite{bigscience2022bloom} & bigscience/bloom-7b1           & 7.1B \\
Llama3-8~\cite{meta2024llama3card}   & meta-llama/Meta-Llama-3-8B     & 8B \\
OPT-13~\cite{zhang2022opt}           & facebook/opt-13b               & 13B \\
Llama-13~\cite{touvron2023llama}     & huggyllama/llama-13b           & 13B \\
Llama2-13~\cite{touvron2023llama2}   & TheBloke/Llama-2-13B-fp16      & 13B \\
\midrule
GPT-4-Turbo        & OpenAI                         & NA \\
GPT-5        & OpenAI                         & NA \\
Gemini-3-Flash       & Google                         & NA \\
GLM-4.5-Flash & ZhipuAI                     & NA \\
\bottomrule
\end{tabular}
\end{table}

\subsection{Hyperparameter Settings}
\label{sec:hyperparameter_setting}
\begin{table}[t]
\centering
\begin{tabular}{lcccc}
\toprule
 & \multicolumn{2}{c}{Uncertainty} & \multicolumn{2}{c}{Uncertainty++} \\
\cmidrule(lr){2-3}\cmidrule(lr){4-5}
 & Black-box & White-box & Black-box & White-box \\
\midrule
Low-probability Percentile & 7   & 7   & 13  & 9 \\
Renyi Entropy Order        & 2.0 & 0.5 & 1.6 & 0.7 \\
Weighting Coefficient      & 0.8 & 0.9 & 0.1 & 0.2 \\
Sample Size                & 100 & 100 & 100 & 100 \\
\bottomrule
\end{tabular}
\caption{Hyperparameter settings for Uncertainty and Uncertainty++.}
\label{tab:hyperparameter_setting}
\end{table}

Following Lastde, we select an optimal set of hyperparameters for each model under both black-box and white-box settings, as shown in Table~\ref{tab:hyperparameter_setting}. However, our hyperparameter analysis in Figure~\ref{fig:hyperparameter} shows that the performance is not sensitive to variations in the R{\'e}nyi entropy order and the weighting coefficient, with fluctuations of less than 1\%.

\section{Details of Main Results}
\subsection{Detail of Main Results in the Black-Box Setting}
\label{sec:black_main}
\begin{table*}[htbp]
\centering
\caption{Experimental results of various detection methods under black-box setting. This table reports detailed results on \textit{Xsum}, \textit{WritingPrompts}, and \textit{Reddit}, corresponding to the main results in Table~\ref{tab:main_black}. All experimental settings follow Appendix~\ref{sec:experimental setup}.}

\label{tab:appendix_black}
\begin{adjustbox}{width=\textwidth}
    \setlength{\tabcolsep}{3pt}
    \footnotesize
    \begin{tabular}{@{}lccccccccccccc@{}}
    \toprule
    \textbf{Method/Models} & \textbf{GPT-2} & \textbf{Neo-2.7} & \textbf{OPT-2.7} & \textbf{Llama-13} & \textbf{Llama2-13} & \textbf{Llama3-8} & \textbf{OPT-13} & \textbf{Bloom-7.1} & \textbf{Falcon-7} & \textbf{Gemma-7} & \textbf{Phi2-2.7} & \textbf{GPT-4-T} & \textbf{Avg.} \\ \midrule
    \rowcolor{gray!20} \multicolumn{14}{c}{\textbf{Xsum}} \\ \midrule
    Likelihood & 54.00 & 48.48 & 63.21 & 54.21 & 54.32 & 98.91 & 68.84 & 39.50 & 54.01 & 65.41 & 54.60 & 60.44 & 59.66 \\
    LogRank & 58.36 & 53.63 & 66.94 & 60.04 & 59.45 & 99.11 & 72.36 & 47.21 & 59.91 & 68.16 & 61.72 & 61.52 & 64.03 \\
    DetectLRR & 67.40 & 66.77 & 71.85 & 71.69 & 68.10 & 96.29 & 75.06 & 65.95 & 71.33 & 70.37 & 74.51 & 61.75 & 71.76 \\
    Lastde & 80.26 & 82.45 & 83.48 & \textbf{71.70} & 67.40 & \textbf{99.37} & 85.89 & \textbf{77.30} & 75.29 & 75.24 & 80.16 & 68.56 & 78.93 \\
    \textbf{Uncertainty} & \textbf{84.88} & \textbf{85.91} & \textbf{88.10} & 70.41 & \textbf{70.38} & 96.90 & \textbf{88.18} & 75.02 & \textbf{77.64} & \textbf{79.03} & \textbf{82.89} & \textbf{68.75} & \textbf{80.67} \\ \hdashline
    DetectGPT & 58.86 & 55.97 & 63.61 & 56.28 & 54.98 & 75.66 & 65.44 & 46.87 & 55.83 & 62.28 & 54.22 & 67.35 & 59.78 \\
    DetectNPR & 55.64 & 50.72 & 63.58 & 56.50 & 55.65 & 96.36 & 66.72 & 44.54 & 56.33 & 60.37 & 55.15 & 62.94 & 60.38 \\
    DNA-GPT & 53.17 & 44.47 & 55.40 & 52.92 & 54.63 & 99.04 & 59.26 & 48.67 & 54.63 & 55.03 & 54.28 & 57.32 & 57.40 \\
    Fast-DetectGPT & 80.81 & 77.25 & 82.10 & 64.83 & 62.79 & \textbf{99.20} & 84.25 & 69.45 & 72.10 & 76.67 & 75.71 & 80.80 & 77.16 \\
    Lastde++ & 88.84 & 89.71 & 90.95 & 73.31 & 73.23 & 98.77 & 91.99 & 82.34 & 82.81 & 83.22 & 85.75 & \textbf{83.12} & 85.34 \\
    \textbf{Uncertainty++} & \textbf{93.39} & \textbf{93.76} & \textbf{94.57} & \textbf{79.63} & \textbf{78.46} & 96.70 & \textbf{93.47} & \textbf{87.52} & \textbf{87.88} & \textbf{85.97} & \textbf{90.50} & 78.76 & \textbf{88.38} \\ \midrule
    \rowcolor{gray!20} \multicolumn{14}{c}{\textbf{WritingPrompts}} \\ \midrule
    Likelihood & 78.53 & 82.34 & 77.08 & 76.44 & 78.98 & 99.79 & 77.43 & 76.75 & 78.60 & 66.21 & 86.18 & 81.49 & 79.99 \\
    LogRank & 81.52 & 85.73 & 82.03 & 79.80 & 81.80 & \textbf{99.92} & 81.39 & 81.38 & 81.47 & 68.11 & 88.73 & 79.02 & 82.58 \\
    DetectLRR & 85.35 & 90.54 & 90.55 & 85.27 & 85.73 & 96.64 & 87.95 & 90.00 & 86.35 & 67.42 & 90.64 & 66.65 & 85.26 \\
    Lastde & 97.20 & 96.54 & \textbf{96.73} & 87.99 & 87.81 & 99.55 & \textbf{94.21} & \textbf{96.72} & 93.45 & 74.21 & 94.04 & 69.05 & 90.63 \\
    \textbf{Uncertainty} & \textbf{97.66} & \textbf{98.48} & 96.36 & \textbf{90.31} & \textbf{90.86} & 98.92 & 93.41 & 96.64 & \textbf{93.89} & \textbf{78.50} & \textbf{96.23} & \textbf{82.04} & \textbf{92.78} \\ \hdashline
    DetectGPT & 77.05 & 78.47 & 83.30 & 74.46 & 77.79 & 90.14 & 84.65 & 70.29 & 80.28 & 62.14 & 84.86 & \textbf{94.21} & 79.80 \\
    DetectNPR & 79.36 & 81.29 & 83.48 & 77.03 & 81.13 & 96.48 & 86.16 & 73.04 & 81.41 & 58.75 & 85.64 & 90.78 & 81.21 \\
    DNA-GPT & 74.27 & 75.82 & 73.55 & 68.35 & 75.52 & \textbf{99.56} & 74.20 & 71.53 & 71.74 & 58.84 & 82.55 & 72.82 & 74.90 \\
    Fast-DetectGPT & 94.93 & 96.07 & 92.02 & 88.25 & 86.43 & 99.20 & 89.64 & 92.77 & 90.05 & 76.51 & 93.65 & 89.88 & 90.78 \\
    Lastde++ & 98.12 & 98.80 & 97.08 & 94.43 & 92.44 & 98.65 & 95.11 & 97.45 & 96.04 & 84.45 & 97.03 & 88.41 & 94.83 \\
    \textbf{Uncertainty++} & \textbf{98.75} & \textbf{99.52} & \textbf{98.91} & \textbf{95.75} & \textbf{94.34} & 98.22 & \textbf{97.36} & \textbf{98.70} & \textbf{97.43} & \textbf{85.44} & \textbf{97.36} & 87.11 & \textbf{95.74} \\ \midrule
    \rowcolor{gray!20} \multicolumn{14}{c}{\textbf{Reddit}} \\ \midrule
    Likelihood & 65.53 & 70.45 & 61.91 & 66.56 & 72.52 & \textbf{100.00} & 60.15 & 69.17 & 69.66 & 78.08 & 81.00 & 97.16 & 74.35 \\
    LogRank & 70.75 & 74.23 & 68.05 & 70.97 & 76.75 & \textbf{100.00} & 65.30 & 73.81 & 73.68 & 80.36 & 83.60 & \textbf{97.19} & 77.89 \\
    DetectLRR & 81.92 & 80.31 & 81.24 & 78.36 & 82.56 & 96.92 & 77.67 & 82.52 & 81.99 & 82.94 & 86.30 & 93.19 & 83.83 \\
    Lastde & 91.89 & 91.32 & 89.90 & 81.61 & 85.45 & 99.89 & 88.99 & 93.79 & 85.03 & 89.19 & 91.51 & 90.32 & 89.91 \\
    \textbf{Uncertainty} & \textbf{95.13} & \textbf{96.33} & \textbf{91.22} & \textbf{85.47} & \textbf{90.03} & 99.46 & \textbf{90.04} & \textbf{95.09} & \textbf{88.32} & \textbf{92.57} & \textbf{94.90} & 94.55 & \textbf{92.76} \\ \hdashline
    DetectGPT & 66.77 & 73.41 & 69.18 & 67.63 & 71.11 & 82.91 & 71.59 & 68.32 & 69.97 & 75.22 & 79.19 & 83.63 & 71.59 \\
    DetectNPR & 69.20 & 73.23 & 72.13 & 69.97 & 75.03 & 97.40 & 72.51 & 71.41 & 73.51 & 78.03 & 81.44 & 86.10 & 76.66 \\
    DNA-GPT & 65.01 & 67.59 & 61.96 & 61.04 & 69.99 & \textbf{99.82} & 63.78 & 65.84 & 68.88 & 73.90 & 79.24 & 82.12 & 71.60 \\
    Fast-DetectGPT & 92.78 & 92.97 & 85.44 & 79.68 & 83.58 & 99.53 & 84.54 & 91.43 & 82.08 & 91.20 & 90.69 & \textbf{93.88} & 88.98 \\
    Lastde++ & 97.15 & 97.91 & 93.87 & 88.05 & 91.29 & 99.38 & 93.37 & 96.91 & 89.95 & 95.11 & 95.85 & 93.30 & 94.35 \\
    \textbf{Uncertainty++} & \textbf{98.39} & \textbf{98.93} & \textbf{96.40} & \textbf{92.59} & \textbf{94.31} & 98.12 & \textbf{95.41} & \textbf{97.46} & \textbf{93.15} & \textbf{96.36} & \textbf{97.18} & 88.91 & \textbf{95.60} \\ \bottomrule
    \end{tabular}
\end{adjustbox}
\end{table*}

Table~\ref{tab:appendix_black} shows that our approach delivers the most consistent performance in the black-box setting. In particular, Uncertainty++ achieves the best average scores on all three datasets, with 88.38 on \textit{Xsum}, 95.74 on \textit{WritingPrompts}, and 95.60 on \textit{Reddit}. It surpasses the strongest baseline, Lastde++ (85.34/94.83/94.35), by 3.04, 0.91, and 1.25 points, respectively, and exceeds Fast-DetectGPT by an even larger margin on average. One important observation is that \textit{Xsum} is substantially more challenging than \textit{WritingPrompts} and \textit{Reddit} for probability-based methods, yet uncertainty-based signals remain effective.

\subsection{Detail of Main Results in the White-Box Setting}
\label{sec:white_main}
\begin{table*}[htbp]
\centering
\caption{Experimental results across \textit{Xsum}, \textit{SQuAD}, and \textit{WritingPrompts} under the white-box setting, corresponding to the main results in Table~\ref{tab:main_white}. All experimental settings follow Appendix~\ref{sec:experimental setup}.}

\label{tab:appendix_white}
\resizebox{0.97\textwidth}{!}{
    \setlength{\tabcolsep}{2pt}
    \footnotesize
    \begin{tabular}{@{}lccccccccccccc@{}}
    \toprule
    \textbf{Method/Models} & \textbf{GPT-2} & \textbf{Neo-2.7} & \textbf{OPT-2.7} & \textbf{GPT-J} & \textbf{Llama-13} & \textbf{Llama2-13} & \textbf{Llama3-8} & \textbf{OPT-13} & \textbf{Bloom-7.1} & \textbf{Falcon-7} & \textbf{Gemma-7} & \textbf{Phi2-2.7} & \textbf{Avg.} \\ \midrule
    
    \rowcolor{gray!20} \multicolumn{14}{c}{\textbf{Xsum}} \\ \midrule
    Likelihood & 86.89 & 86.68 & 81.46 & 81.44 & 61.66 & 63.06 & 99.43 & 77.15 & 85.33 & 68.58 & 67.94 & 88.68 & 79.03 \\
    LogRank & 90.00 & 90.59 & 84.64 & 84.98 & 67.06 & 68.78 & 99.86 & 81.04 & 90.60 & 73.62 & 73.08 & 91.29 & 82.96 \\
    DetectLRR & 92.49 & 92.44 & 85.57 & 86.66 & 79.12 & 77.05 & 98.99 & 84.46 & 95.05 & 79.97 & 81.60 & 91.73 & 87.09 \\
    Lastde & 95.88 & 97.45 & 95.65 & 96.48 & 88.11 & 87.07 & \textbf{100.00} & 93.00 & \textbf{99.32} & 91.33 & 91.47 & 95.81 & 94.30 \\
    \textbf{Uncertainty} & \textbf{97.32} & \textbf{98.23} & \textbf{97.17} & \textbf{98.03} & \textbf{93.97} & \textbf{92.65} & 99.17 & \textbf{95.79} & 98.46 & \textbf{96.12} & \textbf{95.95} & \textbf{97.28} & \textbf{96.68} \\ \hdashline
    DetectGPT & 89.52 & 89.41 & 83.33 & 81.24 & 68.81 & 67.55 & 69.27 & 76.95 & 86.99 & 75.70 & 73.08 & 88.97 & 78.82 \\
    DetectNPR & 90.98 & 92.31 & 86.46 & 86.15 & 69.49 & 70.94 & 98.52 & 80.01 & 93.67 & 77.09 & 73.85 & 92.64 & 84.34 \\
    DNA-GPT & 83.78 & 80.30 & 77.08 & 75.35 & 59.25 & 58.83 & 98.40 & 71.32 & 80.82 & 60.15 & 59.36 & 83.64 & 74.02 \\
    Fast-DetectGPT & 98.91 & 98.96 & 96.98 & 98.12 & 94.77 & 92.65 & \textbf{99.90} & 95.46 & 99.21 & 96.45 & 96.37 & 98.08 & 97.16 \\
    Lastde++ & \textbf{99.34} & 99.76 & 98.48 & 99.03 & 97.83 & 96.31 & 99.77 & 97.67 & \textbf{99.88} & 98.54 & 97.99 & \textbf{98.95} & 98.63 \\
    \textbf{Uncertainty++} & 99.18 & \textbf{99.78} & \textbf{99.13} & \textbf{99.56} & \textbf{98.09} & \textbf{97.22} & 99.69 & \textbf{98.21} & 99.59 & \textbf{99.44} & \textbf{98.84} & 98.28 & \textbf{98.92} \\ \midrule

    \rowcolor{gray!20} \multicolumn{14}{c}{\textbf{SQuAD}} \\ \midrule
    Likelihood & 92.05 & 86.77 & 88.84 & 80.51 & 46.55 & 48.39 & 95.65 & 84.04 & 85.09 & 75.60 & 72.76 & 83.21 & 78.29 \\
    LogRank & 95.14 & 91.48 & 92.39 & 85.84 & 51.87 & 53.59 & 97.22 & 87.98 & 90.56 & 81.19 & 76.95 & 87.35 & 82.63 \\
    DetectLRR & 97.97 & 96.92 & 95.83 & 92.56 & 71.70 & 73.51 & 98.96 & 92.12 & 95.36 & 89.28 & 80.74 & 92.64 & 89.80 \\
    Lastde & 99.62 & 98.82 & 99.00 & 96.20 & 80.24 & 82.40 & \textbf{99.04} & 98.06 & \textbf{99.57} & 96.33 & 91.00 & 96.82 & 94.76 \\
    \textbf{Uncertainty} & \textbf{99.74} & \textbf{98.89} & \textbf{99.24} & \textbf{98.65} & \textbf{82.44} & \textbf{86.20} & 97.34 & \textbf{98.31} & 99.46 & \textbf{97.66} & \textbf{96.36} & \textbf{98.43} & \textbf{96.06} \\ \hdashline
    DetectGPT & 93.73 & 86.26 & 90.15 & 77.83 & 40.96 & 47.16 & 54.71 & 81.76 & 86.81 & 70.52 & 71.46 & 81.66 & 73.58 \\
    DetectNPR & 97.48 & 94.13 & 94.68 & 84.40 & 48.30 & 50.32 & 90.66 & 91.15 & 93.60 & 80.99 & 76.74 & 88.09 & 82.55 \\
    DNA-GPT & 93.74 & 88.78 & 89.96 & 83.33 & 50.65 & 54.83 & 96.56 & 85.43 & 88.22 & 79.94 & 67.68 & 86.44 & 80.46 \\
    Fast-DetectGPT & 99.91 & 99.72 & 99.69 & 99.51 & 87.28 & 89.40 & \textbf{99.91} & 99.65 & 99.75 & 98.83 & 96.78 & 99.28 & 97.48 \\
    Lastde++ & \textbf{99.95} & \textbf{99.93} & 99.94 & 99.72 & 92.77 & 94.19 & 99.85 & \textbf{99.96} & \textbf{99.97} & 99.72 & 98.59 & \textbf{99.30} & 98.66 \\
    \textbf{Uncertainty++} & \textbf{99.95} & 99.66 & \textbf{99.96} & \textbf{99.80} & \textbf{94.27} & \textbf{97.12} & 99.41 & 99.89 & 99.72 & \textbf{99.85} & \textbf{98.77} & 99.18 & \textbf{98.97} \\ \midrule

    \rowcolor{gray!20} \multicolumn{14}{c}{\textbf{WritingPrompts}} \\ \midrule
    Likelihood & 96.03 & 94.76 & 93.93 & 92.94 & 82.77 & 84.64 & 99.94 & 92.17 & 93.61 & 86.12 & 69.56 & 97.13 & 90.30 \\
    LogRank & 97.31 & 96.55 & 95.94 & 95.24 & 87.68 & 88.51 & 99.98 & 94.28 & 96.07 & 89.03 & 74.47 & 97.73 & 92.73 \\
    DetectLRR & 98.02 & 98.85 & 97.99 & 97.52 & 93.54 & 92.37 & 97.22 & 96.43 & 98.58 & 92.93 & 81.68 & 97.95 & 95.26 \\
    Lastde & 99.78 & 99.59 & 99.80 & 98.62 & 97.16 & 96.40 & \textbf{100.00} & 98.84 & \textbf{99.78} & 97.95 & 92.12 & \textbf{98.24} & 98.19 \\
    \textbf{Uncertainty} & \textbf{99.90} & \textbf{99.76} & \textbf{99.67} & \textbf{99.43} & \textbf{98.78} & \textbf{98.81} & 99.57 & \textbf{99.75} & 99.54 & \textbf{98.95} & \textbf{95.79} & 98.13 & \textbf{99.01} \\ \hdashline
    DetectGPT & 97.04 & 95.52 & 97.61 & 92.40 & 81.56 & 81.45 & 86.42 & 96.45 & 94.04 & 87.73 & 67.35 & 98.02 & 89.63 \\
    DetectNPR & 98.85 & 97.88 & 98.57 & 96.04 & 88.00 & 88.23 & 97.48 & 98.19 & 97.59 & 91.11 & 73.62 & \textbf{98.45} & 93.67 \\
    DNA-GPT & 92.25 & 91.32 & 93.34 & 87.95 & 76.95 & 79.72 & 99.26 & 90.77 & 91.19 & 82.04 & 63.85 & 93.91 & 86.88 \\
    Fast-DetectGPT & 99.80 & 99.78 & 99.68 & 99.22 & 98.30 & 98.08 & \textbf{99.88} & 99.53 & 99.80 & 98.50 & 97.50 & 96.94 & 98.92 \\
    Lastde++ & 99.88 & \textbf{99.96} & \textbf{99.95} & 99.77 & 99.51 & 99.39 & 99.84 & 99.73 & \textbf{99.97} & 99.29 & 98.56 & 98.00 & \textbf{99.49} \\
    \textbf{Uncertainty++} & \textbf{99.95} & 99.94 & 99.93 & \textbf{99.87} & \textbf{99.88} & \textbf{99.68} & 99.41 & \textbf{99.88} & 99.95 & \textbf{99.83} & \textbf{99.11} & 94.32 & 99.31 \\ \bottomrule
    \end{tabular}
}
\end{table*}

In the white-box setting (Table~\ref{tab:appendix_white}), most methods operate close to a performance ceiling due to the access to richer internal signals, yet our approach remains among the top performers. Uncertainty++ achieves the highest average on \textit{Xsum} (98.92) and \textit{SQuAD} (98.97), exceeding Lastde++ (98.63 and 98.66) by 0.29 and 0.31 points, respectively, and ranks a close second on \textit{WritingPrompts} with 99.31 compared to 99.49.

\subsection{Performance under Low False Positive Rate}
\label{sec:fpr}

\begin{table}[t]
\centering
\caption{Detection performance on \textit{Xsum}, \textit{SQuAD}, and \textit{WritingPrompts} under the white-box setting. The source models are identical to the main experiments. The best result is in bold, and the second best result is underlined.}
\label{tab:fpr}
\begin{adjustbox}{width=0.8\textwidth}
\begin{tabular}{lcccccc}
\toprule
 & \multicolumn{2}{c}{Xsum} & \multicolumn{2}{c}{SQuAD} & \multicolumn{2}{c}{WritingPrompts} \\
\cmidrule(lr){2-3}\cmidrule(lr){4-5}\cmidrule(lr){6-7}
 & TPR@1\%FPR & TPR@5\%FPR & TPR@1\%FPR & TPR@5\%FPR & TPR@1\%FPR & TPR@5\%FPR \\
\midrule
Fast-DetectGPT & 69.00 & 85.17 & 76.17 & 87.61 & 79.44 & 94.72 \\
Lastde++ & \uline{80.83} & \uline{93.33} & \textbf{84.78} & \uline{92.00} & \uline{91.28} & \uline{97.33} \\
\textbf{Uncertainty++} & \textbf{83.50} & \textbf{94.67} & \uline{83.89} & \textbf{95.22} & \textbf{94.11} & \textbf{99.06} \\
\bottomrule
\end{tabular}
\end{adjustbox}
\end{table}

Following prior work~\cite{masrour-etal-2025-damage}, we report the TPR@1\%FPR and TPR@5\%FPR of Fast-detectgpt, Lastde++, and Uncertainty++ under white-box setting, as a metric of model performance under a low false-positive-rate constraint. 

As shown in Table~\ref{tab:fpr}, Uncertainty++ is overall the most robust method across the three datasets under both low false-positive-rate constraints. On \textit{Xsum} and \textit{WritingPrompts}, Uncertainty++ achieves the highest TPR at both FPR levels. Overall, when the detector is required to keep the false positive rate at 1\% or 5\%, our approach still preserves strong true positive rates, indicating a lower risk of false alarms and more reliable detection under low-FPR operating conditions.

\section{Efficiency Analysis}
\label{sec:efficiency}

\begin{table}[t]
\centering
\caption{Efficiency comparison on the \textit{XSum} dataset under the black-box setting with GPT-J-6B as the proxy model.}
\label{tab:efficiency}
\begin{adjustbox}{width=0.45\columnwidth}
\small
\begin{tabular}{lcc}
\toprule
\multicolumn{3}{c}{\textbf{Probability-based Methods}} \\
\midrule
\textbf{Method} & \textbf{Efficiency (sample/s)} & \textbf{VRAM (GB)} \\
\midrule
DetectLRR & 25.00 & 12.68 \\
Lastde & 25.00 & 12.68 \\
\textbf{Uncertainty} & \textbf{27.27} & 12.68 \\
\midrule
\multicolumn{3}{c}{\textbf{Sampling-based Methods}} \\
\midrule
\textbf{Method} & \textbf{Efficiency (sample/s)} & \textbf{VRAM (GB)} \\
\midrule
Fast-DetectGPT & \textbf{23.08} & 12.68 \\
Lastde++ & 12.50 & 12.68 \\
\textbf{Uncertainty++} & 18.75 & 12.68 \\
\bottomrule
\end{tabular}
\end{adjustbox}
\end{table}

To evaluate computational efficiency, we compare representative detectors on the \textit{XSum} dataset under the black-box setting. All experiments are conducted on a single NVIDIA A100 GPU with 80GB VRAM, and all methods use GPT-J-6B as the proxy model. We report throughput, measured as processed samples per second, and peak GPU memory usage.

As shown in Table~\ref{tab:efficiency}, Uncertainty achieves the highest throughput among all methods, processing 27.27 samples per second, while using the same 12.68 GB VRAM. For sampling-based methods, Uncertainty++ reaches 18.75 samples per second. Overall, the proposed methods achieve competitive efficiency with controllable computational overhead.

\section{Performance under Different Proxy Models}
\label{sec:proxy_model}

\begin{table*}[htbp]
\centering
\caption{Experimental results across different proxy models on the \textit{Reddit} dataset under the black-box setting.}
\label{tab:proxy_model}
\resizebox{0.9\textwidth}{!}{
    \setlength{\tabcolsep}{3pt} 
    \footnotesize
    \begin{tabular}{@{}lccccccccccccc@{}}
    \toprule
    \textbf{Methods/Models} & \textbf{GPT-2} & \textbf{Neo-2.7} & \textbf{OPT-2.7} & \textbf{GPT-J} & \textbf{Llama-13} & \textbf{Llama2-13} & \textbf{Llama3-8} & \textbf{OPT-13} & \textbf{Bloom-7.1} & \textbf{Falcon-7} & \textbf{Phi2-2.7} & \textbf{GPT-4-T} & \textbf{Avg.} \\ \midrule
    
    \rowcolor{gray!15} \multicolumn{14}{c}{\textbf{Gemma-7b}} \\ \midrule
    Fast-DetectGPT & 76.68 & 77.11 & 86.66 & 82.47 & 86.91 & 88.27 & \textbf{98.94} & 87.14 & 83.69 & 86.62 & 87.23 & 48.48 & 82.52 \\
    Lastde++       & \uline{87.84} & \uline{90.48} & \uline{93.55} & \uline{90.92} & \uline{92.88} & \uline{93.24} & \uline{98.24} & \uline{94.72} & \uline{92.65} & \uline{92.83} & \uline{91.96} & \uline{46.62} & \uline{88.83} \\
    Uncertainty++  & \textbf{91.98} & \textbf{93.82} & \textbf{95.37} & \textbf{93.72} & \textbf{93.41} & \textbf{94.83} & 98.11 & \textbf{95.44} & \textbf{93.72} & \textbf{94.85} & \textbf{92.76} & \textbf{62.58} & \textbf{91.72} \\ \midrule

    \textbf{Methods/Models} & \textbf{GPT-2} & \textbf{Neo-2.7} & \textbf{OPT-2.7} & \textbf{GPT-J} & \textbf{Llama-13} & \textbf{Llama2-13} & \textbf{OPT-13} & \textbf{Bloom-7.1} & \textbf{Falcon-7} & \textbf{Gemma-7} & \textbf{Phi2-2.7} & \textbf{GPT-4-T} & \textbf{Avg.} \\ \midrule
    \rowcolor{gray!15} \multicolumn{14}{c}{\textbf{Llama3-8b}} \\ \midrule
    Fast-DetectGPT & 66.62 & 67.78 & 79.79 & 76.70 & 85.81 & 87.66 & 81.39 & 82.30 & 85.50 & 86.93 & 86.46 & 60.25 & 78.93 \\
    Lastde++       & \uline{81.36} & \uline{84.27} & \uline{90.21} & \uline{90.01} & \uline{92.29} & \uline{94.27} & \uline{91.53} & \uline{93.13} & \uline{93.00} & \uline{92.66} & \textbf{93.04} & \uline{60.65} & \uline{88.04} \\
    Uncertainty++  & \textbf{89.10} & \textbf{90.12} & \textbf{92.32} & \textbf{92.87} & \textbf{93.83} & \textbf{95.10} & \textbf{93.54} & \textbf{94.02} & \textbf{94.23} & \textbf{93.19} & \uline{92.13} & \textbf{65.60} & \textbf{90.50} \\ \midrule

    \textbf{Methods/Models} & \textbf{GPT-2} & \textbf{Neo-2.7} & \textbf{OPT-2.7} & \textbf{Llama-13} & \textbf{Llama2-13} & \textbf{Llama3-8} & \textbf{OPT-13} & \textbf{Bloom-7.1} & \textbf{Falcon-7} & \textbf{Gemma-7} & \textbf{Phi2-2.7} & \textbf{GPT-4-T} & \textbf{Avg.} \\ \midrule
    \rowcolor{gray!15} \multicolumn{14}{c}{\textbf{GPT-J-6b}} \\ \midrule
    Fast-DetectGPT & 92.78 & 92.97 & 85.44 & 79.68 & 83.58 & \textbf{99.53} & 84.54 & 91.43 & 82.08 & 91.20 & 90.69 & \textbf{93.88} & 88.98 \\
    Lastde++       & \uline{97.15} & \uline{97.91} & \uline{93.87} & \uline{88.05} & \uline{91.29} & \uline{99.38} & \uline{93.37} & \uline{96.91} & \uline{89.95} & \uline{95.11} & \uline{95.85} & \uline{93.30} & \uline{94.35} \\
    Uncertainty++  & \textbf{98.39} & \textbf{98.93} & \textbf{96.40} & \textbf{92.59} & \textbf{94.31} & 98.12 & \textbf{95.41} & \textbf{97.46} & \textbf{93.15} & \textbf{96.33} & \textbf{97.18} & 88.91 & \textbf{95.60} \\ \bottomrule
    \end{tabular}
}
\end{table*}

Table~\ref{tab:proxy_model} reports performance under the black-box setting when varying the proxy model. We follow the same black-box protocol as the main experiments and keep all evaluation settings fixed, while only replacing GPT-J-6b with Gemma-7b and Llama3-8b. Overall, Uncertainty++ achieves the best average performance under all three proxy choices. Meanwhile, we observe a clear proxy effect that GPT-J-6b yields the strongest results across methods, suggesting that a stronger proxy model provides more reliable conditional probability estimates and therefore raises the performance ceiling.


\end{document}